%% file: paper.tex
\documentclass[runningheads]{llncs}
\usepackage[T1]{fontenc}
\usepackage{graphicx}
\usepackage{booktabs}
\usepackage[misc]{ifsym}

\usepackage{mwe}
\usepackage{amsmath}
\usepackage{amsthm}
\usepackage{amssymb}
\usepackage{amsfonts} 
\usepackage{tikz}
\usepackage{dsfont}
\usepackage{algorithm}
\usepackage[noend]{algorithmic}
\usepackage{hyperref}       % hyperlinks
\usepackage{url}  
\usepackage{subcaption}
\usepackage{ulem}
\usetikzlibrary{decorations.pathreplacing}

\usepackage{changes}
\definechangesauthor[name={LZ}, color=blue]{LZ}

\newtheorem{assumption}{Assumption}
\begin{document}

\title{Beyond Stationarity in Time Series: Discovering Causal Structures and Latent Regimes via Markov Blankets}

\titlerunning{Discovering Causal Structures and Latent Regimes via Markov Blankets}
% If the full title of your paper is short enough to also fit in the running head, you can omit the abbreviated paper title here. You can check as follows: if you comment out the \titlerunning line, something will appear in the header of all odd-numbered pages of your PDF from page 3 onward. This something is either the full title (in which case all is well), or the error message "Title Suppressed Due to Excessive Length". If this error message appears, you're going to want to provide an abbreviated title within the \titlerunning command, because if you won't do it, Springer will do it for you.

%N.B.: Author information (both in the \author{} and \authorrunning{} command) should only be present in the Camera-Ready Version of your paper. The version that you initially submit for review, ought to be double-blind. So, when initially submitting your paper, use:
% \author{Author information scrubbed for double-blind reviewing}

\author{Lei Zan\inst{1} \and
Charles K. Assaad \inst{2} \and
Emilie Devijver \inst{1} \and % \orcidID{0000-1111-2222-3333}} 
Eric Gaussier \inst{1}}

% You may leave out the orcidID information, if you want to.
% Use \corr to indicate the corresponding author. Note the spacing around the \corr command. Only one author can be the corresponding author.

%N.B.: comment out the \authorrunning{} command for the double-blind version of your paper submitted for review. Later, if your paper is accepted, use the command for the Camera-Ready Version.

\authorrunning{L. Zan et al.}

\institute{Univ. Grenoble Alpes, CNRS, Grenoble INP, LIG, Grenoble, France \email{xdzanlei@gmail.com, 
\{emilie.devijver, eric.gaussier\}@univ-grenoble-alpes.fr}
\and
 Sorbonne Université, INSERM, Institut Pierre Louis d’Epidémiologie et de Santé Publique, F75012, Paris, France
\email{charles.assaad@inserm.fr}}

\maketitle              % typeset the header of the contribution

\begin{abstract}
This paper introduces Regime-aware Constraint-Based and Noise-Based causal discovery with Markov Blankets (RCBNB-MB), a novel causal discovery algorithm for time series that relaxes the common assumption of a single, time-consistent causal structure. Time series are typically observed at discrete time points and often exhibit regime changes that challenge the assumption of a static causal structure, a limitation in many real-world dynamic systems. To address this challenge, RCBNB-MB identifies latent causal regimes, defined as subsets of time points within which a stable causal structure holds. The algorithm follows an iterative strategy that segments the time series into regimes and discovers the causal graph within each regime. By leveraging the Markov blanket rather than direct parents, RCBNB-MB gains robustness to errors in causal discovery and preserves predictive information. 
We provide theoretical guarantees for RCBNB-MB’s ability to recover both regime transitions and causal graphs under reasonable assumptions. Furthermore, we validate its effectiveness through extensive experiments on simulated datasets with known ground truth and real-world IT monitoring data, where taking into account regime shifts is critical. 
Empirical results show that RCBNB-MB systematically outperforms baseline approaches in accurately detecting regime changes and their associated causal graphs, positioning it as a robust and versatile framework for non-stationary time series analysis.

\end{abstract}

\input{Introduction}
\input{Sota}

\input{Preliminaries}

\input{Methodology}
\input{Experiments}
\input{Conclusion}

% \section{Recent Advances from the Mariana Trench}

% Displayed equations are centered and set on a separate
% line.
% \begin{equation}
% x + y = z
% \end{equation}
% Please try to avoid rasterized images for line-art diagrams and
% schemas. Whenever possible, use vector graphics instead (see
% Fig.\@ \ref{fig1}).

\begin{credits}
% \subsubsection{\ackname} A bold run-in heading in small font size at the end of the paper is
% used for general acknowledgments, for example: This study was funded
% by X (grant number Y).

% \subsubsection{Generative AI Disclosure}
% The authors used LLM-based tools solely for rephrasing sentences to improve readability, and all scientific content remains the authors' own.

% \subsubsection{\discintname}
%  The authors have no competing interests to declare that are
% relevant to the content of this article.

% It is now necessary to declare any competing interests or to specifically
% state that the authors have no competing interests. Please place the
% statement with a bold run-in heading in small font size beneath the
% (optional) acknowledgments,
% for example: The authors have no competing interests to declare that are
% relevant to the content of this article. Or: Author A has received research
% grants from Company W. Author B has received a speaker honorarium from
% Company X and owns stock in Company Y. Author C is a member of committee Z.
\end{credits}
%
% ---- Bibliography ----
%
% BibTeX users should specify bibliography style 'splncs04'.
% References will then be sorted and formatted in the correct style.
%
% \bibliographystyle{splncs04}
% \bibliography{mybibliography}
%% Note that this preceding line implies that you store your BibTeX references in a file called 'mybibliography.bib'. If you instead store your references in a file with a different name, for instance 'references.bib', the preceding line should read '\bibliography{references}'. Whatever you do, DO NOT put the file name extension .bib inside the \bibliography command; this will trip up LaTeX compilers. 

\bibliographystyle{splncs04}
\bibliography{references}

\newpage
\appendix
\input{Appendix}

\end{document}

%% file: Introduction.tex
\section{Introduction}
\label{sec:intro}

Causal discovery for time series  is a fundamental problem in many domains, such as IT monitoring systems, climate science, economics, epidemiology, and neuroscience. Understanding causal relationships over time enables better prediction, diagnosis, and intervention in complex dynamic systems. However, most causal discovery methods assume stationarity, i.e., that causal dependencies remain constant. In practice, this assumption is often unrealistic: many systems undergo regime shifts, where causal influences change due to unobserved factors.

To address this, we propose a method that automatically segments time series into distinct regimes, within which stationarity holds. By jointly learning the segmentation and the causal structure in each regime, our approach provides a more flexible and robust framework for causal discovery in heterogeneous time series.
A key challenge in this setting is that regime assignment must be inferred from data. Standard methods typically rely on the causal parent set (PA) to group time points into regimes, but errors in causal graph estimation can lead to incorrect assignments. To mitigate this, we instead use the Markov blanket~\cite{pellet2008using} (MB), a set comprising a variable's parents, children, and spouses that renders the variable independent of all others when conditioned upon. As a theoretically optimal feature set for prediction and classification~\cite{tsamardinos2003uniquness}, the MB provides a more stable and informative representation for regime assignment.
This is particularly relevant in applications where regime changes are implicit and unknown. For instance, in IT monitoring data \cite{aïtbachir2023case}, system behavior differs between normal and anomalous states, but the transition times are often unknown. Similar shifts can arise from hidden external factors in neuroscience, epidemiology, and climate science. Our method learns these regimes directly from data while recovering their causal structures.

Our main contributions are:
\begin{itemize}
    \item Identifying a key challenge in heterogeneous time series: when instantaneous effects exist (e.g., due to low sampling rates), standard causal discovery methods struggle to correctly assign timestamps to the right regime.
\item Reframing the regime assignment problem as a prediction task, leveraging the Markov blanket to improve robustness against errors in graph estimation.
\item Introducing RCBNB-MB, a novel method that simultaneously identifies regimes and reconstructs regime-specific causal graphs, including instantaneous connections.
\item Demonstrating the effectiveness of our approach in uncovering regime-dependent causal structures through an empirical validation on synthetic and real-world IT monitoring data.
\end{itemize}

The paper is organized as follows. Section~\ref{sec:SOTA} reviews existing causal discovery methods for heterogeneous time series and their limitations. Section~\ref{sec:pre} introduces our causal model. Section~\ref{sec:theo} presents our proposed method, RCBNB-MB. Section~\ref{sec:exper} provides experimental results on synthetic and real-world data. Finally, Section~\ref{sec:conclu} concludes the paper.

%% file: Sota.tex
\section{Related work}
\label{sec:SOTA}
Causal discovery for time series has been extensively reviewed, with a focus on consistency over time \cite{assaad2022survey,10.1145/3705297}. In this paper, we focus on cases where consistency is violated, i.e., when causal relationships evolve over time or across regimes. We summarize the main characteristics in Table \ref{tab:methods_comparison}.

Classically, methods for causal discovery in non-stationary time series assume the availability of predefined regime or context information. \cite{huang2019causal} introduced a causal discovery method for heterogeneous time series using nonlinear state-space models with linear causal modeling. \cite{huang2020causal} developed the CD-NOD framework, which models changes in causal mechanisms across regimes as being confounded by an unobserved pseudo-confounder. This confounding is addressed using a surrogate variable that captures these changes and reveals additional causal directions through the Independent Causal Mechanisms principle. However, CD-NOD does not preserve temporal information, providing only a summarized causal graph. CD-NOTS \cite{2025_sadeghi} assumes causal stationarity in the conditional distribution, resulting in a common causal graph for the entire time series. Extensions by \cite{ferdous2023b} recover temporal relationships, while \cite{yang2025} reconstructs the causal structure using copula entropy, assuming the surrogate variable influences other variables at specific lags—a strong assumption that may not hold in many real-world scenarios.
Similarly, J(oint)-PCMCI$^+$ \cite{gunther2023causal} incorporates both observed and unobserved context variables, akin to surrogate variables, and preserves temporal relationships across regimes. JIT-LiNGAM \cite{fujiwara2023causal} learns local, linear approximations for nonlinear, heterogeneous time series but does not account for temporal lags.

\begin{table}[t]
\centering
\caption{Comparison of causal discovery methods for non-stationary time series. PA and MB denote the parent set and Markov blanket, respectively.}
\label{tab:methods_comparison}
\resizebox{\textwidth}{!}{%
\begin{tabular}{|l|c|c|c|c|c|}
\hline
\textbf{Method} & \textbf{Regimes} &  \textbf{Temporal} & \textbf{Contemp.} & \textbf{Fixed} &  \textbf{Prediction} \\ 
& \textbf{} & \textbf{Lags} & \textbf{Effects} & \textbf{Causal Order}  & \\\hline
CD-NOD \cite{huang2020causal} & Segmentation  & \checkmark & \checkmark & \texttimes &  \texttimes\\ \hline
% CD-NOTS \cite{2025_sadeghi} &  \texttimes & \checkmark & \checkmark & \texttimes &  \\ \hline
% (e)CDANs \cite{ferdous2023b,ferdous2023a} & \texttimes  & \checkmark & \checkmark & \texttimes &  \\ \hline
% CE-CDN \cite{yang2025} & \texttimes  & \checkmark & \texttimes & \texttimes  & \\ \hline
J(oint)-PCMCI$^+$ \cite{gunther2023causal} & Input  & \checkmark & \checkmark & \texttimes &  \texttimes\\ \hline
% JIT-LiNGAM \cite{fujiwara2023causal} & \texttimes & \texttimes & \texttimes & \texttimes & \texttimes & \checkmark & \texttimes & \\ \hline
LoSST \cite{kummerfeld2013tracking} & Mahalanobis distance  & \texttimes & \texttimes & \texttimes  &\texttimes \\ \hline
SPACETIME \cite{Mameche_2025} & Learnt & \checkmark & \checkmark & \checkmark  & PA \\ \hline
% PCMCI$_\Omega$ \cite{gao2024causal} & \texttimes & \texttimes & \checkmark & \texttimes & \texttimes & \checkmark & \texttimes&  \\ \hline
SDCI \cite{rodas2022causal} & Learnt  & \checkmark &  \texttimes & \texttimes  & PA \\ \hline
RPCMCI \cite{saggioro2020reconstructing} & Learnt  & \checkmark & \texttimes & \texttimes  & PA \\ \hline
CASTOR \cite{2025_RahmaniF} & Learnt  & \checkmark & \checkmark & \texttimes  & PA \\ \hline
\textbf{RCBNB-MB} & Learnt  & \checkmark & \checkmark & \texttimes  & MB \\ \hline
\end{tabular}%
}
\end{table}

More useful for real-world datasets is the ability to infer regimes and recover causal structures without assuming predefined context variables. Some methods operate directly on the time series. \cite{kummerfeld2013tracking} introduced a real-time method for updating causal graphs by identifying change points using the Mahalanobis distance. However, this approach assumes i.i.d. data, which is rarely true for time series. \cite{Mameche_2025} introduced a method to detect regimes, contexts, and causal structures in multiple multivariate time series, but it assumes a fixed causal order and skeleton across contexts—a strong assumption that may not hold if causal relationships change direction between regimes. \cite{gao2024causal} addressed a special case of heterogeneity in time series, assuming that different causal mechanisms occur sequentially and periodically over time. This represents a strong assumption that may not hold in many real-world scenarios. Another viewpoint is to construct regimes to improve forecasting.  They are all (to our knowledge)  based on forecasting with causal parents. 
\cite{rodas2022causal,2022_cai} introduced a strategy to identify both causal relationships and states in state-dependent stationary time series. However, they do not account for instantaneous connections between observed time series. Similarly, \cite{saggioro2020reconstructing} proposed RPCMCI, which detects regimes in heterogeneous time series and reconstructs causal graphs for each regime. However, this method focuses solely on time-lagged relationships, neglecting contemporaneous causal effects. \cite{2025_RahmaniF} proposed CASTOR, which learns a DAG for each regime while determining the number of regimes and their sequential arrangement, based on an EM algorithm.
If all those methods discover a causal graph per regime under different assumptions, they are discriminating regimes by using a forecasting model within the time series based on the causal parents. Our method, on the contrary, relies on the Markov blanket.

%% file: Preliminaries.tex
\section{Causal modeling for heterogeneous time series}
\label{sec:pre}

\begin{figure}[t]
	\begin{minipage}[t]{1\linewidth}
    \scalebox{0.9}{
		\centering
		\begin{tikzpicture}[{black, circle, draw, inner sep=0}]
			\tikzset{nodes={draw,rounded corners},minimum height=0.7cm,minimum width=0.7cm, font=\scriptsize}
			\tikzset{latent/.append style={fill=gray!30}}

                %%%%%%%%%%%%%%%%%%%%%%%%%%%%%%%%%%%%%%%%%%%%%%%%%%
                % One regime  
                %%%%%%%%%%%%%%%%%%%%%%%%%%%%%%%%%%%%%%%%%%%%%%%%%%
			
			\node (X-2_1) at (0,-1) {$X_{t-9}$} ;
			\node (X-1_1) at (1.5,-1) {$X_{t-8}$};
			\node (X_1) at (3,-1) {$X_{t-7}$};
			\node (Z-2_1) at (0,0) {$Y_{t-9}$} ;
			\node (Z-1_1) at (1.5,0) {$Y_{t-8}$};
			\node (Z_1) at (3,0) {$Y_{t-7}$};
			\node (Y-2_1) at (0,1) {$Z_{t-9}$} ;
			\node (Y-1_1) at (1.5,1) {$Z_{t-8}$};
			\node (Y_1) at (3,1) {$Z_{t-7}$};
			
			\draw[->,>=latex] (X-2_1) -- (X-1_1);
			\draw[->,>=latex] (X-1_1) -- (X_1);
			\draw[->,>=latex] (Z-2_1) -- (Z-1_1);
			\draw[->,>=latex] (Z-1_1) -- (Z_1);
			\draw[->,>=latex] (Y-2_1) -- (Y-1_1);
			\draw[->,>=latex] (Y-1_1) -- (Y_1);
			
			\draw[->,>=latex] (Z-2_1) -- (X_1);
			\draw[->,>=latex] (Z_1) -- (Y_1);
			\draw[->,>=latex] (Z-1_1) -- (Y-1_1);
			\draw[->,>=latex] (Z-2_1) -- (Y-2_1);
			
			\node (X1_1) at (4.5,-1) {$X_{t-6}$} ;
			\node (Z1_1) at (4.5,0) {$Y_{t-6}$} ;
			\node (Y1_1) at (4.5,1) {$Z_{t-6}$} ;

			\draw[->,>=latex] (X_1) -- (X1_1);
			\draw[->,>=latex] (Z_1) -- (Z1_1);
			\draw[->,>=latex] (Y_1) -- (Y1_1);
			
			\draw[->,>=latex] (Z1_1) -- (Y1_1);
			\draw[->,>=latex] (Z-1_1) -- (X1_1);

			\coordinate[left of=X-2_1] (d1_1);
			\draw [dashed,>=latex] (X-2_1) to[left] (d1_1);
			\coordinate[left of=Z-2_1] (d1_1);
			\draw [dashed,>=latex] (Z-2_1) to[left] (d1_1);
			\coordinate[left of=Y-2_1] (d1_1);
			\draw [dashed,>=latex] (Y-2_1) to[left] (d1_1);
			
			\coordinate[right of=X1_1] (d1_1);
			\draw [dashed,>=latex] (X1_1) to[right] (d1_1);
			\coordinate[right of=Z1_1] (d1_1);
			\draw [dashed,>=latex] (Z1_1) to[right] (d1_1);
			\coordinate[right of=Y1_1] (d1_1);
			\draw [dashed,>=latex] (Y1_1) to[right] (d1_1);

                %\draw[decorate,decoration={brace,amplitude=10pt, mirror}] (-1,-1.5) -- (5.5,-1.5);
                %\node[draw=none, font=\scriptsize] at (2.25,-2.1) {$\mathcal{G}^1_f$};

                %%%%%%%%%%%%%%%%%%%%%%%%%%%%%%%%%%%%%
                % regime index 
                %%%%%%%%%%%%%%%%%%%%%%%%%%%%%%%%%%%%%
                \node[latent, draw=none, rectangle, minimum width=0.3cm] (C-2_1) at (0,2.4) {$C_{t-9}{ = 1}$};
			\node[latent, draw=none, rectangle, minimum width=0.3cm] (C-1_1) at (1.5,2.4) {$C_{t-8}{ = 1}$};
			\node[latent, draw=none, rectangle, minimum width=0.3cm] (C_1) at (3,2.4) {$C_{t-7}{ = 1}$};
                \node[latent, draw=none, rectangle, minimum width=0.3cm] (C1_1) at (4.5,2.4) {$C_{t-6}{ = 1}$};

                \draw [dashed,>=latex] (-1,1.7) to[right] (12,1.7);

         %       \coordinate[left of=C-2_1] (d1_2);
			%\draw [dashed,>=latex] (C-2_1) to[left] (d1_2);
			
			\coordinate[right of=C1_1] (d1_2);
			\draw [dashed,>=latex] (C1_1) to[right] (d1_2);

                % regime 1
                \path[->]  (C-2_1) edge [bend left=25] (X-2_1);
                \path[->]  (C-2_1) edge [bend left=25] (Y-2_1);
                \path[->]  (C-2_1) edge [bend left=25] (Z-2_1);
                \path[->]  (C-1_1) edge [bend left=25] (X-1_1);
                \path[->]  (C-1_1) edge [bend left=25] (Y-1_1);
                \path[->]  (C-1_1) edge [bend left=25] (Z-1_1);
                \path[->]  (C_1) edge [bend left=25] (X_1);
                \path[->]  (C_1) edge [bend left=25] (Y_1);
                \path[->]  (C_1) edge [bend left=25] (Z_1);
                \path[->]  (C1_1) edge [bend left=25] (X1_1);
                \path[->]  (C1_1) edge [bend left=25] (Y1_1);
                \path[->]  (C1_1) edge [bend left=25] (Z1_1);

            %%%%%%%%%%%%%%%%%%%%%%%%%%%%%%%%%%%%%%%%%%%%%%%%%%
            % another regime  
            %%%%%%%%%%%%%%%%%%%%%%%%%%%%%%%%%%%%%%%%%%%%%%%%%%
            
                \node (X-2_2) at (6.5,-1) {$X_{t-2}$} ;
			\node (X-1_2) at (8,-1) {$X_{t-1}$};
			\node (X_2) at (9.5,-1) {$X_{t}$};
			\node (Z-2_2) at (6.5,0) {$Y_{t-2}$} ;
			\node (Z-1_2) at (8,0) {$Y_{t-1}$};
			\node (Z_2) at (9.5,0) {$Y_{t}$};
			\node (Y-2_2) at (6.5,1) {$Z_{t-2}$} ;
			\node (Y-1_2) at (8,1) {$Z_{t-1}$};
			\node (Y_2) at (9.5,1) {$Z_{t}$};
			
			\draw[->,>=latex] (X-2_2) -- (X-1_2);
			\draw[->,>=latex] (X-1_2) -- (X_2);
			\draw[->,>=latex] (Z-2_2) -- (Z-1_2);
			\draw[->,>=latex] (Z-1_2) -- (Z_2);
			\draw[->,>=latex] (Y-2_2) -- (Y-1_2);
			\draw[->,>=latex] (Y-1_2) -- (Y_2);
			
			\draw[->,>=latex] (Z-2_2) -- (Y_2);
			\draw[->,>=latex] (X_2) -- (Z_2);
			\draw[->,>=latex] (X-1_2) -- (Z-1_2);
			\draw[->,>=latex] (X-2_2) -- (Z-2_2);
			
			\node (X1_2) at (11,-1) {$X_{t+1}$} ;
			\node (Z1_2) at (11,0) {$Y_{t+1}$} ;
			\node (Y1_2) at (11,1) {$Z_{t+1}$} ;

			\draw[->,>=latex] (X_2) -- (X1_2);
			\draw[->,>=latex] (Z_2) -- (Z1_2);
			\draw[->,>=latex] (Y_2) -- (Y1_2);
			
			\draw[->,>=latex] (X1_2) -- (Z1_2);
			\draw[->,>=latex] (Z-1_2) -- (Y1_2);

			\coordinate[left of=X-2_2] (d1_2);
			\draw [dashed,>=latex] (X-2_2) to[left] (d1_2);
			\coordinate[left of=Z-2_2] (d1_2);
			\draw [dashed,>=latex] (Z-2_2) to[left] (d1_2);
			\coordinate[left of=Y-2_2] (d1_2);
			\draw [dashed,>=latex] (Y-2_2) to[left] (d1_2);
			
			\coordinate[right of=X1_2] (d1_2);
			\draw [dashed,>=latex] (X1_2) to[right] (d1_2);
			\coordinate[right of=Z1_2] (d1_2);
			\draw [dashed,>=latex] (Z1_2) to[right] (d1_2);
			\coordinate[right of=Y1_2] (d1_2);
			\draw [dashed,>=latex] (Y1_2) to[right] (d1_2);

                %\draw[decorate,decoration={brace,amplitude=10pt, mirror}] (5.5,-1.5) -- (12,-1.5);
                % \node[inner sep=0pt,outer sep=0pt,font=\scriptsize]  {Regime 1};
               % \node[draw=none, font=\scriptsize] at (8.75,-2.1) {$\mathcal{G}^2_f$};

                 %%%%%%%%%%%%%%%%%%%%%%%%%%%%%%%%%%%%%
                % regime index 
                %%%%%%%%%%%%%%%%%%%%%%%%%%%%%%%%%%%%%

                \node[fill=cyan!30,draw=none, rectangle, minimum width=0.3cm] (C-2_1) at (6.5,2.4) {$C_{t-2}=2$};
			\node[fill=cyan!30,draw=none, rectangle, minimum width=0.3cm] (C-1_1) at (8,2.4) {$C_{t-1}=2$};
			\node[fill=cyan!30,draw=none, rectangle, minimum width=0.3cm] (C_1) at (9.5,2.4) {$C_{t}=2$};
                \node[fill=cyan!30, draw=none, rectangle, minimum width=0.3cm] (C1_1) at (11,2.4) {$C_{t+1}=2$};

               \coordinate[left of=C-2_1] (d1_2);
			\draw [dashed,>=latex] (C-2_1) to[left] (d1_2);
			
	%		\coordinate[right of=C1_1] (d1_2);
	%		\draw [dashed,>=latex] (C1_1) to[right] (d1_2);
   
                 % regime 1
                \path[->]  (C-2_1) edge [bend left=20] (X-2_2);
                \path[->]  (C-2_1) edge [bend left=20] (Y-2_2);
                \path[->]  (C-2_1) edge [bend left=20] (Z-2_2);
                \path[->]  (C-1_1) edge [bend left=20] (X-1_2);
                \path[->]  (C-1_1) edge [bend left=20] (Y-1_2);
                \path[->]  (C-1_1) edge [bend left=20] (Z-1_2);
                \path[->]  (C_1) edge [bend left=15] (X_2);
                \path[->]  (C_1) edge [bend left=15] (Y_2);
                \path[->]  (C_1) edge [bend left=15] (Z_2);
                \path[->]  (C1_1) edge [bend left=20] (X1_2);
                \path[->]  (C1_1) edge [bend left=20] (Y1_2);
                \path[->]  (C1_1) edge [bend left=20] (Z1_2);

		\end{tikzpicture}
            }
            % \captionsetup{justification=centering}
             \caption*{(a) Full time causal graph $\mathcal{G}_f$ from two distinct latent regimes.}
	\end{minipage}	
      %  \vspace{1.5cm}
	\begin{minipage}[t]{\linewidth}
		\centering
		\scalebox{0.9}{
		\begin{tikzpicture}[{black, circle, draw, inner sep=0}]
			\tikzset{nodes={draw,rounded corners},minimum height=0.7cm,minimum width=0.7cm, font=\scriptsize}
			\tikzset{latent/.append style={fill=gray!30}}

                %%%%%%%%%%%%%%%%%%%%%%%%%%%%%%%%%%%
                % one regime 
                %%%%%%%%%%%%%%%%%%%%%%%%%%%%%%%%%%%
			
			\node (X2_1) at (0,-1) {$X_{t-2}$} ;
			\node (X1_1) at (1.5,-1) {$X_{t-1}$};
			\node (X_1) at (3,-1) {$X_{t}$};
			\node (Z2_1) at (0,0) {$Y_{t-2}$} ;
			\node (Z1_1) at (1.5,0) {$Y_{t-1}$};
			\node (Z_1) at (3,0) {$Y_{t}$};
			\node (Y2_1) at (0,1) {$Z_{t-2}$} ;
			\node (Y1_1) at (1.5,1) {$Z_{t-1}$};
			\node (Y_1) at (3,1) {$Z_{t}$};
			\draw[->,>=latex] (X2_1) -- (X1_1);
			\draw[->,>=latex] (X1_1) -- (X_1);
			\draw[->,>=latex] (Z2_1) -- (Z1_1);
			\draw[->,>=latex] (Z1_1) -- (Z_1);
			\draw[->,>=latex] (Y2_1) -- (Y1_1);
			\draw[->,>=latex] (Y1_1) -- (Y_1);
			
			\draw[->,>=latex] (Z2_1) -- (X_1);
			\draw[->,>=latex] (Z_1) -- (Y_1);
			\draw[->,>=latex] (Z1_1) -- (Y1_1);
			\draw[->,>=latex] (Z2_1) -- (Y2_1);

            \node[latent, draw=none, rectangle, font=\scriptsize] at (-1,0) {$\mathcal{G}^1_w$:};

                % \draw[decorate,decoration={brace,amplitude=10pt, mirror}] (0,-1.5) -- (3,-1.5);
                % % \node[inner sep=0pt,outer sep=0pt,font=\scriptsize]  {Regime 1};
                % \node[draw=none, font=\scriptsize] at (1.5,-2.1) {$\mathcal{G}^1_w$};
                
                %%%%%%%%%%%%%%%%%%%%%%%%%%%%%%%%%%%%%
                % regime index 
                %%%%%%%%%%%%%%%%%%%%%%%%%%%%%%%%%%%%%
   %              \node[draw, rectangle, minimum width=1.2cm] (C-2_2) at (0,2.4) {$C_{t-2}=1$};
			% \node[draw, rectangle, minimum width=1.2cm] (C-1_2) at (1.5,2.4) {$C_{t-1}=1$};
			% \node[draw, rectangle, minimum width=1.2cm] (C_2) at (3,2.4) {$C_{t}=1$};

   %              \draw [dashed,>=latex] (-1,1.7) to[right] (10,1.7);

                % regime 1
                % \path[->]  (C-2_2) edge [bend left=25] (X2_1);
                % \path[->]  (C-2_2) edge [bend left=25] (Y2_1);
                % \path[->]  (C-2_2) edge [bend left=25] (Z2_1);
                % \path[->]  (C-1_2) edge [bend left=25] (X1_1);
                % \path[->]  (C-1_2) edge [bend left=25] (Y1_1);
                % \path[->]  (C-1_2) edge [bend left=25] (Z1_1);
                % \path[->]  (C_2) edge [bend left=17] (X_1);
                % \path[->]  (C_2) edge [bend left=17] (Y_1);
                % \path[->]  (C_2) edge [bend left=17] (Z_1);

                %%%%%%%%%%%%%%%%%%%%%%%%%%%%%%%%%%%
                % another regime 
                %%%%%%%%%%%%%%%%%%%%%%%%%%%%%%%%%%%
                
                \node (X2_2) at (6,-1) {$X_{t-2}$} ;
			\node (X1_2) at (7.5,-1) {$X_{t-1}$};
			\node (X_2) at (9,-1) {$X_{t}$};
			\node (Z2_2) at (6,0) {$Y_{t-2}$} ;
			\node (Z1_2) at (7.5,0) {$Y_{t-1}$};
			\node (Z_2) at (9,0) {$Y_{t}$};
			\node (Y2_2) at (6,1) {$Z_{t-2}$} ;
			\node (Y1_2) at (7.5,1) {$Z_{t-1}$};
			\node (Y_2) at (9,1) {$Z_{t}$};
			\draw[->,>=latex] (X2_2) -- (X1_2);
			\draw[->,>=latex] (X1_2) -- (X_2);
			\draw[->,>=latex] (Z2_2) -- (Z1_2);
			\draw[->,>=latex] (Z1_2) -- (Z_2);
			\draw[->,>=latex] (Y2_2) -- (Y1_2);
			\draw[->,>=latex] (Y1_2) -- (Y_2);
			
			\draw[->,>=latex] (Z2_2) -- (Y_2);
			\draw[->,>=latex] (X_2) -- (Z_2);
			\draw[->,>=latex] (X1_2) -- (Z1_2);
			\draw[->,>=latex] (X2_2) -- (Z2_2);
   
                %\draw[decorate,decoration={brace,amplitude=10pt, mirror}] (6,-1.5) -- (9,-1.5);
                % \node[inner sep=0pt,outer sep=0pt,font=\scriptsize]  {Regime 1};
                \node[fill=cyan!30, draw=none, rectangle, font=\scriptsize] at (5,0) {$\mathcal{G}^2_w$:};

                %%%%%%%%%%%%%%%%%%%%%%%%%%%%%%%%%%%%%
                % regime index 
                %%%%%%%%%%%%%%%%%%%%%%%%%%%%%%%%%%%%%
   %              \node[draw, rectangle, minimum width=1.2cm] (C-2_2) at (6,2.4) {$C_{t-2}=2$};
			% \node[draw, rectangle, minimum width=1.2cm] (C-1_2) at (7.5,2.4) {$C_{t-1}=2$};
			% \node[draw, rectangle, minimum width=1.2cm] (C_2) at (9,2.4) {$C_{t}=2$};
 
                % % regime 2
                % \path[->]  (C-2_2) edge [bend left=25] (X2_2);
                % \path[->]  (C-2_2) edge [bend left=25] (Y2_2);
                % \path[->]  (C-2_2) edge [bend left=25] (Z2_2);
                % \path[->]  (C-1_2) edge [bend left=25] (X1_2);
                % \path[->]  (C-1_2) edge [bend left=25] (Y1_2);
                % \path[->]  (C-1_2) edge [bend left=25] (Z1_2);
                % \path[->]  (C_2) edge [bend left=25] (X_2);
                % \path[->]  (C_2) edge [bend left=25] (Y_2);
                % \path[->]  (C_2) edge [bend left=25] (Z_2);
		\end{tikzpicture}
        }
	   \caption*{(b) Window causal graphs $\mathcal{G}_w^1$ (left) and  $\mathcal{G}_w^2$ (right) from two distinct regimes.}
	\end{minipage}
	\caption{Two graphic causal relation representations for time series with two distinct regimes: (a) full time causal graphs from two distinct regimes (b) window causal graphs from two distinct regimes. The variable $C_t$ is latent and has to be inferred. }
	\label{fig:ftcg_wcg}
\end{figure}

Consider a dynamic system represented by a multivariate time series of dimension $d$, observed over $\mathbf{T}$, and denoted as $\{\mathbf{X}_{t}\}_{t \in  \mathbf{T}} = \{X_{t}^{j}\}_{t \in  \mathbf{T}, j \in \{1, \dots, d\}}$. These data exhibit complex causal dependencies, naturally modeled using a Directed Acyclic Graph (DAG), called the full-time causal graph, denoted as $\mathcal{G}_{f}=(\mathbf{V}_{f}, \mathbf{E}_{f})$. Here, $\mathbf{V}_{f}$ represents vertices corresponding to time-indexed variables, and $\mathbf{E}_{f}$ captures directed causal relationships.
A key aspect of our framework is the introduction of hidden context variables $C_t$, where $C_t \in \{1, \dots, r\}$ indicates the regime to which each time point $t$ belongs. The number of distinct regimes is denoted as $r$.
We do not assume regimes are known a priori. Instead, $C_t$ acts as an unobserved confounder, directly influencing all observed variables at time $t$ and defining the regime-specific causal structure.
%\textcolor{red}{position wrt SOTA} 
%We assume that $C_t$ directly influences all observed variables at time $t$, meaning that the system's causal structure depends on the underlying regime. To formalize this, we introduce \textbf{semi-pseudo causal sufficiency}, inspired by \textbf{pseudo causal sufficiency} \cite{huang2020causal}.
\begin{assumption}[Semi-pseudo causal sufficiency]
\label{ap:pseudo_causal_sufficiency}
All hidden confounding between observed variables is captured by $(C_t)_{t\in \mathbb{T}}$. 
%There is only one unobserved confounder, which correspond to the regime index. %Under semi-pseudo causal sufficiency, potential unobserved confounders can be expressed as functions of the regime index. \textcolor{red}{to be discussed}
\end{assumption}
%This assumption helps manage hidden confounding effects while maintaining causal interpretability.
To infer the causal graph, we assume \textbf{consistency throughout time} \cite{assaad2022survey} (also known as \textbf{causal stationarity} \cite{runge2018causal}) within each regime: causal relationships remain invariant over time, allowing us to sidestep the intractable full causal graph $\mathcal{G}_{f}$ and instead operate on window causal graphs $\mathcal{G}_{w}^k = (\mathbf{V}_{w}, \mathbf{E}_{w}^k)$ for $k\in \{1,\ldots,r\}$.  While this assumption breaks down at regime boundaries (where structural shifts occur) it provides a principled way to balance model complexity and interpretability.
%This assumption ensures stable causal relationships over time within a given regime, enabling a more compact representation. Of course, stability is broken at regime boundaries. 
%Consistency throughout time allows one to avoid working with $\mathcal{G}_{f}$ and focus on the \textbf{window causal graph} for each regime, denoted as $\mathcal{G}_{w}^k = (\mathbf{V}_{w}, \mathbf{E}_{w}^k)$ for $k\in \{1,\ldots,r\}$. 
These graphs capture causal dependencies within a finite window of length $\tau_{\max}$ which corresponds to the maximum lag between direct causes and effects. %, summarizing local causal mechanisms.
This leads to the following model.
\begin{definition}[Regime-based multivariate time series]
    \label{def:hetero_time_series}
    A multivariate time series, $\{\mathbf{X}_{t}\}_{t \in  \mathbf{T}}$, is called a \emph{regime-based multivariate time series} if it can be partitioned into $r > 1$ disjoint regimes encoded by $C_t \in \{1,\ldots, r\}$ for $t \in  \mathbf{T}$, with a specific window causal graph $\mathcal{G}_w^{k}$ in regime $k$ s.t. $\forall k_1, k_2 \in \{1,\dots, r\}$, $k_1 \neq k_2 \Rightarrow \mathcal{G}_w^{k_1} \neq \mathcal{G}_w^{k_2}$. 
    % s. t. $\forall k_1, k_2 \in \{1,\dots, r\}, \mathcal{G}_w^{k_1} \neq \mathcal{G}_w^{k_2}$.
\end{definition}
The structural causal model (SCM)~\cite{pearl2009causality} for the $j^{th}$ variable in a heterogeneous multivariate time series is given by:
\begin{equation}\label{DSCM}
   X_{t}^{j} = \sum_{k =1}^{r}\mathds{1}_{C_{t}=k} \times g^{j,C_{t}}(\mathbf{Pa}(X_{t}^{j};\mathcal{G}^{C_{t}}_w), \epsilon_{t}^{j}) \text{ for $t \in \mathbf{T}$}.
\end{equation}
Here, $\mathds{1}_{C_{t}=k}$ is an indicator function that equals $1$ if and only if $C_{t}=k$. The response function $g^{j,C_{t}}(\cdot)$ is deterministic within each regime and depends only on $C_{t}$. $\mathbf{Pa}(X_{t}^{j};\mathcal{G}^{C_{t}}_w)$ represents the causal parents of $X_{t}^{j}$ in $\mathcal{G}^{C_{t}}_w$, including lagged and instantaneous variables. The noise terms $(\epsilon_{t}^{j})_{t \in  \mathbf{T}, 1\leq j \leq d}$ are assumed to be independent. Note that we allow for instantaneous causal relations in the window causal graph of each regime, which is often forgotten in the literature (except in methods such as CASTOR~\cite{2025_RahmaniF}). A regime does not necessarily correspond to a single contiguous time segment and can span multiple non-adjacent intervals.  %\textcolor{orange}{each of length greater than $2\tau_{\max}$ (accounting for dependencies in both past and future time steps)}. 
Figure \ref{fig:ftcg_wcg}(a) provides an example of a full-time causal graph for a multivariate time series with two distinct regimes, while Figure \ref{fig:ftcg_wcg}(b) depicts the corresponding window causal graphs $\mathcal{G}^1_{w}$ and $\mathcal{G}^2_{w}$.

%To align the causal graph with the observational joint distribution of time series for each regime, we assume the \textbf{Causal Markov Condition}~\cite{spirtes2001causation,Pearl_2000} and \textbf{Faithfulness}~\cite{spirtes2001causation}. When using constraint-based methods, we can typically identify causal graphs up to their Markov equivalence class (MEC), which includes all DAGs encoding the same conditional independence structure.

%% file: Methodology.tex
\section{Causal discovery from heterogeneous  time series}
\label{sec:theo}

\begin{figure}
    \includegraphics[width=\textwidth, trim = 0.75cm 0 0.75cm 0, clip=TRUE]{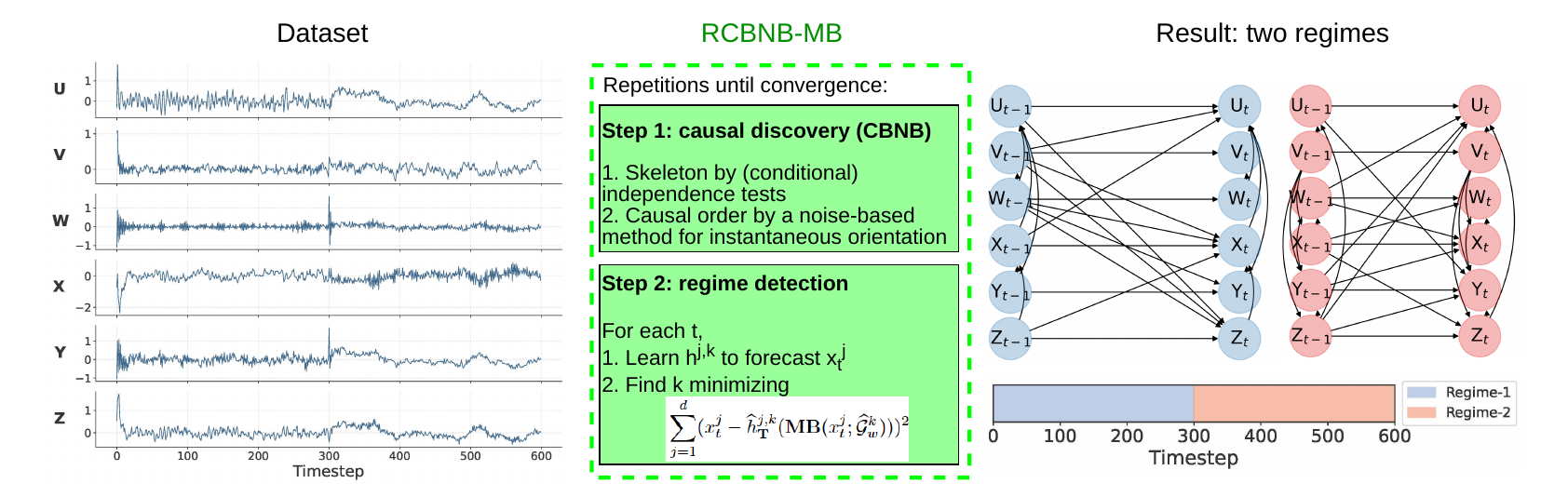}
    \caption{\textbf{Illustration of our method.} Left: an example of a dataset with 6 time series, which violates the causal stationarity assumption. Middle: flowchart of our iterative method through the two steps: causal discovery and regime detection. Right: results of our method on this dataset:  two regimes are considered, and we display the two causal graphs and the regime assignment.}
    \label{fig:1}
\end{figure}
The goal of causal discovery from heterogeneous, observational multivariate time series is twofold: 
\begin{itemize}
    \item[1.] Reconstruct the causal graph $(\mathcal{G}^{k}_w)$ within each regime $k$,
    \item[2.] Assign each timestamp to the correct regime.
\end{itemize}

We propose to address those two problems by an alternating approach, denoted as RCBNB-MB (Regime-aware Constraint-Based and Noise-Based causal discovery with Markov Blankets), which %, starting from a given assignment of time instants to regimes, 
alternates between two steps by first constructing the window causal graph of each regime from given assignments of time instants to regimes, and second by re-assigning time instants to regimes according to the new window causal graphs. We detail these two steps below, and the full procedure is illustrated in Figure \ref{fig:1}. In the following, we assume that the number of distinct regimes $r$ is known. 

\subsection{Regime-Aware Window Causal Graph Discovery}
RCBNB-MB begins by, given the regime index, discovering a causal graph, which reduces to time-series causal discovery within each regime. While any causal discovery algorithm could be integrated, we adopt CBNB~\cite{bystrova2024hybrids}, a {hybrid} method that prunes edges with a constraint-based approach and orients them with a noise-based one, thereby retaining the strengths of both families: the graph is fully oriented and the orientation search is restricted to the discovered skeleton, which improves its efficiency.
More specifically, given the regime index for each time point, CBNB recovers the window causal graph $\mathcal{G}_w$ in two  steps:
\begin{enumerate}
    \item a constraint-based step, which infers the skeleton of the window causal graph using (conditional) independence tests and orients lagged edges using temporal ordering,
    \item a noise-based step that orients the remaining instantaneous edges: within each group of variables linked by undirected edges, a restricted noise-based algorithm estimates a causal order among the instantaneous variables, including the lagged variables as covariates in the regressions to control for confounding.
\end{enumerate}

CBNB’s advantages are threefold: it requires only adjacency faithfulness ~\cite{ramsey2006adjacency}  (not full faithfulness), recovers the true causal graph (not just its Markov class), and outperforms pure constraint/noise-based methods in small-sample settings  \cite{malinsky2018causal,aïtbachir2023case}. CBNB's assumptions are formalized below.

\begin{assumption}[Adjacency Faithfulness]
% \begin{restatable}{assumption}{assumptionone}
\label{ap:adjacency_faithfulness}
   Let  $\mathcal{G}^{k}_w =(\textbf{V}_{w}, \textbf{E}^{k}_w)$ be the causal graph within any regime $k$. If two nodes $X$ and $Y$ in $\textbf{V}_{w}$ are adjacent in $\mathcal{G}^{k}_w$, then they are dependent conditionally on any subset of $\textbf{V}_{w} \backslash \{X, Y \}$. 
% \end{restatable}
\end{assumption}

\begin{assumption}[Identifiable functional model]
\label{ap:identifiable_model}
Given a heterogeneous multivariate time series satisfying the SCM in Equation~\eqref{DSCM}, each function $g^{j,k}(\cdot)$ in regime $k \in \{1,\ldots, r\}$ for a component $j \in \{1,\ldots, d\}$ belongs to an identifiable functional model class, as defined in \cite{peters2011identifiability}.
\end{assumption}

% As shown in \cite[Theorem 2]{bystrova2024hybrids}, given the regime index of each timepoint, under \textbf{causal sufficiency}, \textbf{causal Markov condition}, \textbf{adjacency faithfulness}, \textbf{consistency throughout time}, and Assumption~\ref{ap:identifiable_model}, CBNB correctly recovers the window causal graph, provided perfect conditional independence information.

%\textcolor{orange}{JE NE SUIS PAS CONVAINCU DE LA FORMULATION QUI SUIT It follows directly from \cite[Theorem 2]{bystrova2024hybrids} that given the \emph{correct} regime index of each timepoint, under Assumptions~\ref{ap:pseudo_causal_sufficiency}, \ref{ap:adjacency_faithfulness},  Assumption~\ref{ap:identifiable_model}, and \textbf{consistency throughout time}, CBNB correctly recovers the causal graph $\mathcal{G}^k_w$ for each regime $k$, provided perfect conditional independence information  and if different timepoints of each regime form a consecutive sequence. Theoretically, the algorithm is designed to operate on a consecutive sequence of time points, ensuring that causal relationships are correctly captured without temporal gaps. However, in practice, we also apply it to sequences with gaps, acknowledging that while the underlying assumptions may not strictly hold in such cases, the algorithm can still provide meaningful insights.} 
By \cite[Theorem 2]{bystrova2024hybrids},  CBNB correctly recovers the causal graph $\mathcal{G}^k_w$ for each regime $k$ if regime indexes are correct,  Assumptions~\ref{ap:pseudo_causal_sufficiency}, \ref{ap:adjacency_faithfulness},  Assumption~\ref{ap:identifiable_model}, and {consistency throughout time} hold, and conditional independence tests are perfect. While CBNB assumes consecutive time points, we extend its use to gapped regimes (Section~\ref{subsec:real_data}).

\subsection{Assigning Time Points to Regimes Using Markov-Blanket-based Prediction}

Assuming the window causal graph for each regime is known, we assign each time point to the regime that minimizes the prediction error for all variables in  $\mathbf{X}_{t} = \{\mathbf{X}_{t-\tau_{\max}}, \dots, \mathbf{X}_{t+\tau_{\max}}\}$. % the set of all variables $X_{t'}^{j}, \, 1 \le j \le d$ such that $t-\tau_{\max} \le t' \le t+\tau_{\max}$. 
% where $\tau_{\max}$ defines the maximum lag in the window causal graph. 
%For any variable $X_t^j$, the optimal predictor $h_*^{j,k}$ depends only on its Markov Blanket $\mathbf{MB}(X_{t}^{j};\mathcal{G}_w^{C_t}$ \cite{tsamardinos2003uniquness}. 
%The set of its values will be denoted by $\mathbf{x}_{t}$. By definition, $\mathbf{X}_{t}$ is guaranteed to contain the Markov blankets of all variables $X^j_t$. %$\mathbf{X}_{-t}^{-j}$, respectively $\mathbf{x}_{-t}^{-j}$, will denote $\mathbf{X}_{t} \setminus X_{t}^{j}$, respectively $\mathbf{x}_{t} \setminus x_{t}^{j}$. The Markov blanket of $X_{t}^{j}$ in the causal graph $\mathcal{G}^k_w$, where $k \in \{1,\dots,r\}$, is denoted as $\textbf{MB}(X_t^j; \mathcal{G}^k_w)$, and its values are denoted as $\textbf{MB}(x_t^j; \mathcal{G}^k_w)$. 
%Furthermore, we assume in this section that the window causal graphs for the different regimes $\{\mathcal{G}_w^{k}\}_{k \in \{1, \dots, r\}}$ are  known.
%
Predicting the value of a variable \(X^j_t\) given the other variables, \(\mathbf{X}_{-t}^{-j} = \mathbf{X}_{t} \setminus X_{t}^{j}\), traditionally involves minimizing the expected squared error, and the set of potential covariates can be restricted to the Markov blanket of \(X_t^j\) in the causal graph \(\mathcal{G}^k_w\), where \(k \in \{1, \dots, r\}\), as established by \cite{tsamardinos2003uniquness}: for a class of functions $\mathcal{H}$, 
\begin{equation}\label{eq:minimal_b}
    \underset{h^{j,C_t}\in \mathcal{H}} {\operatorname{argmin}} \mathop{\mathbb{E}}\left[ (X^j_t - h^{j,C_t}(\mathbf{X}_{-t}^{-j}))^2\right]  =\underset{h^{j,C_t}\in \mathcal{H}} {\operatorname{argmin}}  \mathop{\mathbb{E}}\left[ (X^j_t - h^{j,C_t}(\mathbf{MB}(X_{t}^{j};\mathcal{G}_w^{C_t})))^2\right].
\end{equation}
%
%where $h^{j,C_t} \in \mathcal{H}$ denotes a predictor from the class of functions $\mathcal{H}$. 
Since the Markov blanket is identical for all instances of \(X^j\) within a given regime, we leverage this structure to approximate the optimal predictor \(h^{j,k}_{*}\), where \(1 \le k \le r\), using empirical risk minimization:
\begin{equation}\label{eq:emp-version}
    \widehat{h}^{j,k}_{\mathbf{T}} (\mathbf{MB}(x_t^j;\mathcal{G}_w^k))= \underset{h^{j,C_t} \in \mathcal{H}} {\operatorname{argmin}} \sum_{t \in \mathbf{T}} \mathds{1}_{C_t=k} \left( x_t^j - h^{j,C_t}(\mathbf{MB}(x_t^j;\mathcal{G}_w^{C_t})) \right)^2.
\end{equation}

%where \(\mathcal{H}\) is the function class used for approximation. 
Using consistency results similar to the ones developed in \cite{gyorfi2002distribution} then leads to the following theorem, the proof of which is given in Appendix~\ref{appendix:th}, which provides a theoretical criterion for deciding the regime for a given variable at a particular time instant.

% \begin{restatable}{theorem}{primetheorem}\label{th:regime-criterion}
%     Suppose that \(C_t=k\), that the function \(h^{j,k}_{*}(\mathbf{MB}(x_t^j;\mathcal{G}_w^k))\) approximates \(x_t^j\) within a bounded error, i.e., \(\exists M> 0, \forall x_t^j, |x_t^j - h^{j,k}_{*}(\mathbf{MB}(x_t^j;\mathcal{G}_w^k))| < M\), and that the estimate $\widehat{h}^{j,k}_{\mathbf{T}}$ converges to the ''best'' function $h^{j,k}_{*}$ in the following sense:
% %
%  \begin{equation} \label{eq:univ-approx}
%      \lim_{k(\mathbf{T}) \rightarrow +\infty} %\mathop{\mathbb{E}} \left[ 
%      \int_{\mathbf{S}} \left( \widehat{h}^{j,k}_{\mathbf{T}}(\mathbf{S}) - h^{j,k}_{*}(\mathbf{S}) \right)^2 dP(\mathbf{S}) %\right]
%      = 0.
%  \end{equation} Then, for any other regime \(\ell \neq k\):
%     \[
%     \lim_{k(\mathbf{T}) \to +\infty} \mathop{\mathbb{E}}\left[ (X^j_t - \widehat{h}^{j,k}_{\mathbf{T}}(\mathbf{MB}(X_t^{j};\mathcal{G}_w^k)))^2\right] \le \lim_{k(\mathbf{T}) \to +\infty} \mathop{\mathbb{E}}\left[ (X^j_t - \widehat{h}^{j,\ell}_{\mathbf{T}}(\mathbf{MB}(X_t^{j};\mathcal{G}_w^{\ell})))^2\right].
%     \]
% \end{restatable}

\begin{theorem}\label{th:regime-criterion}
    Suppose that \(C_t=k\), that the function \(h^{j,k}_{*}(\mathbf{MB}(x_t^j;\mathcal{G}_w^k))\) approximates \(x_t^j\) within a bounded error, i.e., \(\exists M> 0, \forall x_t^j, |x_t^j - h^{j,k}_{*}(\mathbf{MB}(x_t^j;\mathcal{G}_w^k))| < M\), and that the estimate $\widehat{h}^{j,k}_{\mathbf{T}}$ converges to the ''best'' function $h^{j,k}_{*}$ in the following sense:
 \begin{equation} \label{eq:univ-approx-main-text}
     \lim_{k(\mathbf{T}) \rightarrow +\infty} %\mathop{\mathbb{E}} \left[ 
     \int_{\mathbf{S}} \left( \widehat{h}^{j,k}_{\mathbf{T}}(\mathbf{S}) - h^{j,k}_{*}(\mathbf{S}) \right)^2 dP(\mathbf{S}) %\right]
     = 0,
 \end{equation} where $\mathbf{S}$ is any subset of $\mathbf{X}$. Then, for any other regime \(\ell \neq k\):
    \[
    \lim_{k(\mathbf{T}) \to +\infty} \mathop{\mathbb{E}}\left[ (X^j_t - \widehat{h}^{j,k}_{\mathbf{T}}(\mathbf{MB}(X_t^{j};\mathcal{G}_w^k)))^2\right] \le \lim_{k(\mathbf{T}) \to +\infty} \mathop{\mathbb{E}}\left[ (X^j_t - \widehat{h}^{j,\ell}_{\mathbf{T}}(\mathbf{MB}(X_t^{j};\mathcal{G}_w^{\ell})))^2\right].
    \]
\end{theorem}

\noindent Note that Eq.~\ref{eq:univ-approx-main-text} typically holds for universal approximators under conditions on the class of functions considered \cite{gyorfi2002distribution}.

Under the assumption of sufficient sample size in each regime and bounded prediction error, Theorem~\ref{th:regime-criterion} implies that time \(t\) should be assigned to regime \(k\) if, for all \(\ell \neq k\):

\begin{equation}\label{crit1}
\mathop{\mathbb{E}}\left[ (X^j_t - \widehat{h}^{j,k}_{\mathbf{T}}(\mathbf{MB}(X_t^{j};\mathcal{G}_w^k)))^2\right] < \mathop{\mathbb{E}}\left[ (X^j_t - \widehat{h}^{j,\ell}_{\mathbf{T}}(\mathbf{MB}(X_t^{j};\mathcal{G}_w^{\ell})))^2\right].
\end{equation}

Equality in Equation~\ref{crit1} can arise in three cases: (i) the Markov blankets are identical, (ii) one Markov blanket is a strict superset of the other, or (iii) alternative variable sets provide equivalent prediction accuracy. The following assumption\footnote{This assumption is reasonable as it simply states that different regimes likely lead to different Markov blankets for at least some variables, and that if a set differs from the true Markov blanket without being a superset of it, it will likely yield a prediction different from the one of the true Markov blanket.} helps resolve such ambiguities:

\begin{assumption}\label{ap:different_MB}
    Given a multivariate heterogeneous time series \(\{\mathbf{X}_{t}\}_{t \in  \mathbf{T}}\):
    \begin{description}
         \item[(i)] For any two distinct regimes \(k_1, k_2 \in \{1,\dots, r\}\), there exists a variable \(X^j_t\) such that \(\mathbf{MB}(X^j_t; \mathcal{G}_w^{k_1}) \neq \mathbf{MB}(X^j_t; \mathcal{G}_w^{k_2})\).
         \item[(ii)] For any regime \(k\) and variable \(X^j_t\), if a set of variables \(\mathbf{S}\) is not a superset of %\sout{\(\subset \mathbf{X}\) excludes}
         the true Markov blanket, then the expectation in Theorem~\ref{th:regime-criterion} using \(\mathbf{S}\) differs from the one using \(\mathbf{MB}(X^j_t; \mathcal{G}_w^k)\).
    \end{description}
\end{assumption}

Under Assumption~\ref{ap:different_MB}, if multiple regimes yield the same prediction error, we break ties by selecting the regime with the smallest Markov blanket containing a variable satisfying (i). This leads to the final regime assignment criterion:

\noindent \textbf{Criterion t2r (time to regime):} \textit{Assign \(C_t\) to the regime that minimizes:}
\begin{equation}\label{eq:finalcrit}
\sum_{j =1}^{d} (x^j_t - \widehat{h}^{j,k}_{\mathbf{T}}(\mathbf{MB}(x_t^{j};\widehat{\mathcal{G}}_w^k)))^2.
\end{equation}
\textit{If multiple regimes minimize this criterion, select the one with the smallest Markov blanket over all variables satisfying Assumption~\ref{ap:different_MB}(i).}

The procedure we propose can be summarized as follows: for each regime $k$ and each dimension $j$, a predictor is learned using the Markov blanket of $X^j_t$. Then, a sample $\{x_t^j\}_{j \in \{1, \dots, d\}}$ is assigned to the regime whose predictors yield the lowest overall prediction error. In the case of ties, the regime with the smallest Markov blanket is selected.

 \subsection{RCBNB-MB: an algorithm for causal discovery from multiple regimes}
 \label{sec:method}
 
Finally, the overall problem can be framed as the following optimization task: determine the assignment vector $\mathbf{C} = [C_{t_1}, \dots, C_{t_2}]^{T}$, where each coordinate belongs to $\{1,\dots, r\}$,  the regime-specific window causal graphs $\{\mathcal{G}_w^k\}_{1\leq k \leq r}$ and the functions $\mathbf{H}=\{\widehat{h}^{j,k}_{\mathbf{T}}\}$ for $1 \le j \le d$ and $1 \le k \le r$ that minimize:

\begin{equation}\label{eq:empirical_risk_MB} L_{emp}(\mathbf{x}_{t}; \mathbf{C}, \mathbf{H}, \{\mathcal{G}_w^k\}_{1\leq k \leq r}) = \sum_{j=1}^{d} \sum_{t \in \mathbf{T}} \left({x}^j_{t} - \sum_{k=1}^{r}\mathds{1}_{C_{t}=k} \widehat{h}^{j,k}_{\mathbf{T}}(\mathbf{MB}(x^j_t; \mathcal{G}_w^{k}))\right)^2. 
\end{equation}

To ensure meaningful regime assignments, we furthermore impose the following constraints\label{ass:constraints}: 

\begin{enumerate}
\item Each time point belongs to exactly one regime,
\item Each regime contains at least $N_{\ell}$ observations over the entire time span, with $N_{\ell} \ge (2\tau_{\max}+1)d$, so as to be able to infer it using Markov blankets, 
%\footnote{We impose to have at least $2\tau_{\max}+1$ time instants for each time series in any given regime to be sure to capture the complete Markov blanket of at least one time instant, as required by our approach.}
and at most $N_c$ transitions to prevent overly fragmented assignments.
\end{enumerate}

%
%%\begin{assumption}%[Constraints on the optimization problem]
%%\label{ass:constraints} 
%%We assume that: 
%%\begin{itemize} 
%%\item Each time point belongs to exactly one regime. 
%%\item Each regime must contain at least $N_{\ell}$ observations over the entire time span. \item Each regime must have at least $N_c$ transitions to prevent overly fragmented assignments. 
%%\end{itemize} 
%%\end{assumption}

To solve the optimization problem in Equation~\eqref{eq:empirical_risk_MB} under those constraints, RCBNB-MB proceeds as follows: first, it assigns time points randomly to regimes with $\mathbf{C}^{[0]}$, ensuring that each regime spans at least $N_{\ell}$  timestamps. %This guarantees that even if $\mathcal{G}_w^{k}$ is fully connected, $\widehat{h}^{j,k}_{\mathbf{T}}$ remains solvable. 
%This provides an initial assignment $\mathbf{C}^{[0]}$. 
Then, it alternates between the two steps:
\begin{enumerate}
    \item[Step 1] Given the current assignment $\mathbf{C}^{[\text{ite}]}$, estimate the causal graph $\mathcal{G}_w^k$ for each regime $k$ using CBNB~\cite{bystrova2024hybrids}.
Then, for each variable $X^j_t$, determine its Markov blanket $\mathbf{MB}({X}^j_{t}; \mathcal{G}_w^k)$ and update the function set $\mathbf{H}^{[ite]}$ by minimizing Equation~\eqref{eq:empirical_risk_MB}.
\item[Step 2] Keeping $\mathbf{H}^{[\text{ite}]}$ fixed, update the assignment vector $\mathbf{C}^{[\text{ite+1}]}$ using Criterion $t2r$, subject to constraints. 
This is a standard optimization problem with linear constraints that can be efficiently solved using existing solvers.
\end{enumerate}

To avoid poor convergence, RCBNB-MB runs $N_i$ random initializations and selects the solution with the lowest empirical risk after $N_o$ iterations.

Assignments near regime transitions may be less accurate due to limited past data, but their impact is negligible when regime changes are sparse.

If regime assignments are correct, Step 1 recovers the true window causal graphs (and Markov blankets) as $k(\mathbb{T}) \rightarrow +\infty$ (Theorem~\ref{th:regime-criterion}, Assumption~\ref{ap:different_MB}(ii)). Step 2 reduces the squared error in $L_{emp}$ if observed values are close to their expectations.
% {

While there is no theoretical guarantee that RCBNB-MB consistently converges to the ground truth, empirical results in Section~\ref{sec:exper} show that its outputs are often close to the ground truth in the majority of cases.
%} To mitigate the risk of convergence to poor local minima, RCBNB-MB is executed for $N_i$ different initializations of $\mathbf{C}^{[0]}$, and the best performance according to the empirical risk given in Eq. \eqref{eq:empirical_risk_MB} is kept. For each initialization, the process is iterated up to $N_o$ times. 
 The pseudocode of the algorithm is presented in Appendix~\ref{appendix:pseudo_code}, with the corresponding code available in the supplementary materials.  %Algorithm~\ref{alg:regime_CBNBW}. 

%% file: Experiments.tex
\section{Experiments}
\label{sec:exper}
% To evaluate the accuracy of our method in assigning timestamps to the correct regimes and in reconstructing the window causal graph within each regime, we designed the following experiments. 
We designed the following experiments to evaluate our method's accuracy in regime assignment and causal graph reconstruction.

\paragraph{Baselines} {Our framework, denoted RCBNB-MB, combines CBNB for causal discovery with the Markov blanket (MB) for regime detection. To assess the impact of these choices, we benchmark against alternative causal discovery methods, including VarLiNGAM~\cite{hyvarinen2010estimation}, Dynotears~\cite{pamfil2020dynotears}, PCMCI~\cite{runge2019detecting}, and PCMCI$^+$~\cite{runge2020discovering}, each paired with either the Markov blanket (MB) or direct parents (PA) as the predictive variable set, following the naming convention R\{method\}-\{variable set\}. Notably, RPCMCI-PA corresponds to the method in~\cite{saggioro2020reconstructing}, while RPCMCI-MB is its MB-based variant. We also include CASTOR~\cite{2025_RahmaniF}, the random baseline NegControl~\cite{2025_Petersen}, CD-NOD~\cite{huang2020causal}, and J-PCMCI$^+$~\cite{gunther2023causal} for comparison, while SPACETIME~\cite{Mameche_2025} is excluded due to its high computational cost. Since NegControl, CD-NOD, and J-PCMCI$^+$ cannot detect regimes within the time series, their MER values are not reported. All methods were implemented based on publicly available Python libraries, as described in Appendix~\ref{appendix:source_code}.}

\paragraph{Evaluation metrics} 
% The performance of assigning timestamps to the correct regimes was assessed by calculating the average Mean Error Rate (MER) across all datasets for each length of the time series. In the implementation of the method, the assignment vector $\mathbf{C}$ is represented using one-hot encoding, resulting in a binary matrix of size $r \times |\mathbf{T}|$, denoted as $\mathbf{C}_{bina} = \{C_{bina,t}^{k}\}_{t \in \mathbf{T}, k \in \{1, \dots, r\}}$. Each row corresponds to a specific regime, and each column represents a time instance. In this encoding, $C_{bina, t}^{k} = 1$ indicates that the time instance at $t$ belongs to regime $k$, while $C_{bina, t}^{k} = 0$ indicates that it does not belong to regime $k$ at that time instance. Furthermore, the method's output does not determine the correct order of the regimes relative to the ground truth. While it segments the time series into distinct regimes, it cannot align these separated regimes with the ground truth regimes. Therefore, all possible enumerations of the regimes in the output are considered, and the smallest Mean Error Rate (MER) is selected as the result. For a given enumeration, the MER is computed as follows:

% The performance of reconstructing the window causal graph is then evaluated by calculating the F1-score (oriented), comparing the reconstructed causal graph with the true underlying causal graph within each regime.
We evaluate two aspects of performance: (1) the accuracy of regime assignment and (2) the quality of causal graph reconstruction.

Regime assignment accuracy is measured using the Mean Error Rate (MER), which quantifies the proportion of incorrectly assigned timestamps, up to regime label mismatch. Formally, given the one-hot encoded assignment matrices $\mathbf{C}_{bina}$ (predicted) and $\mathbf{C}^*_{bina}$ (ground truth), MER is defined as
% \begin{equation} \label{eq:MER} MER = \underset{\pi \text{ permutation}}{\operatorname{min}} \frac{| \pi(\mathbf{C}_{bina}) - \mathbf{C}_{bina}^{*}|_{r \times |\mathbf{T}|}}{2|\mathbf{T}|}, \end{equation}
$\mathrm{MER}=\min_{\pi}|\pi(\mathbf{C}_{bina})-\mathbf{C}^{*}_{bina}|_{r \times |\mathbf{T}|}/(2|\mathbf{T}|)$, where $|\cdot|_{r \times |\mathbf{T}|}$ denotes the element-wise $L_1$ norm and $\pi$ is a permutation over rows.

Causal graph reconstruction is evaluated using two complementary metrics: the oriented F1-score, which compares the estimated window causal graphs with the ground-truth graphs within each regime (higher is better), and the normalized Structural Hamming Distance (lower is better), defined as $nSHD = \frac{1}{r}\sum_{k=1}^{r} SHD(\widehat{\mathcal{G}}_w^k, \mathcal{G}_w^k)/|\boldsymbol{E}_w^k|$, where $SHD(\cdot,\cdot)$ counts missing, spurious, and reversed instantaneous edges, and $|\boldsymbol{E}_w^k|$ is the number of edges in the ground-truth graph $\mathcal{G}_w^k$, nSHD can thus exceed 1. We compare with a random graph to get a negative control, as suggested in \cite{2025_Petersen}.

\subsection{Simulated data} 
\label{subsec:simulated_data}
\paragraph{Simulation setup}
%We focused on linear causal relationships among six variables, of length 600 (similar results with length 1200 are given in appendix). 
%However, it is important to note that our proposed method can also be extended to handle 
%Each regime's time instances are consecutive and divide the entire time series equally. For each time series length, we generate 50 datasets. We set $\tau_{\max} = 1$, considering contemporaneous dependencies for each variable. The window causal graph in each regime is generated randomly. Each variable includes a self-causal relationship with a lag of 1, while other edges appear with a probability of 0.3. Measures are implemented to limit the number of common edges between the window causal graphs of the two regimes to no more than 7. Additionally, according to Assumption~\ref{ap:different_MB}, at least one variable's Markov blanket differs between each pair of regimes.
We consider linear causal relationships among six variables over a time series of length 600 (results for length 1,200 are provided in Appendix~\ref{appendix:simu_1200}, and results for a scenario with 15 variables in Appendix~\ref{appendix:15_variables}), with 2 and 3 regimes. 
Each regime consists of consecutive time points, evenly dividing the time series (results for unequal regime sizes are provided in Appendix~\ref{appendix:unequal_regime_size}). Time points at regime boundaries are independent, with the first point of each regime sampled from noise. For each time series length, we generate 50 datasets. We set $\tau_{\max} = 1$, allowing for contemporaneous dependencies. For each regime, the window causal graph is generated from the Erdos-Renyi~\cite{erdds1959random} model: each variable has a self-causal link at lag 1, and all other edges appear with probability 0.3. To enforce regime differences, we constrain the number of shared edges between the causal graphs of two regimes to at most 7 and guarantee that at least one variable has a different Markov blanket across regimes. % (see Constraints~\ref{ap:different_MB}).

Within each regime, the data follow the structural causal model below:
\begin{equation}
	\label{eq:linear_SCM}
	X_{t}^{j} = \sum_{X_{t-\tau}^{i} \in \mathbf{Pa}(X_{t}^{j})}b_{i,j,\tau}X_{t-\tau}^{i} + \epsilon_{t}^{j},
\end{equation}
where the coefficients are sampled as $b_{i,j,\tau} \sim \mathbf{U}((-0.9,-0.5) \cup (0.5,0.9))$ and the noise term follows $\epsilon_{t}^{j} \sim \mathbf{U}(-0.1, 0.1)$. To prevent extreme values, if any observation within a regime exceeds 100, the dataset is regenerated.

\paragraph{Method configuration} 
%{For each method, the regime assignment is initialized by iterating over each regime and timestep, randomly assigning each timestep to a regime with a probability of 0.95. Initially, this process results in a significant overlap between regimes. However, due to Assumption~\ref{ass:constraints}, the optimization process ensures that, in the final assignment, each timestep belongs to only one regime.} We fix the number of initialization points $N_i=50$, the number of maximal optimization iterations $N_o=20$, the number of regimes $r=2$, the minimum size per regime $N_{\ell}=7$ (which is greater than $2\tau_{\max}$), ensuring that we can obtain a unique solution of estimators within each regime, the maximum transition points per regime $N_c=3$, the function class of $\mathbf{H}$ is settled as a linear function. For CBNB-w, it uses PCMCI$^{*}$ in the constraint-based part and VarLiNGAM in the noise-based part. We set the significance threshold \(\alpha\) as 0.1 for all (conditional) independence tests. For other causal discovery methods, VarLiNGAM and Dynotears, we use the default parameters. 
The regime assignment is initialized by randomly assigning each timestep to a regime with a probability of 0.95, initially creating overlaps. However, our constraints %Assumption~\ref{ass:constraints}
ensure that the final assignment is unique for each timestep.
We set the number of initialization points to $N_i=50$, the maximum optimization iterations to $N_o=20$, the number of regimes to $r=2$, and the minimum regime size to $N_{\ell}=18$ ($(2\tau_{\max}+1)d$), ensuring estimator identifiability (see Appendix~\ref{appendix:para_robustness} for a parameter robustness analysis of RCBNB-MB). Each regime allows up to $N_c=3$ transitions. The function class $\mathbf{H}$ is chosen as linear.
$\tau_{\min}$ is set to 0 and $\tau_{\max}$ to 1. For CBNB, PCMCI$^{+}$ is used in the constraint-based part and VarLiNGAM in the noise-based part. The significance threshold $\alpha$ is set to 0.1 for all (conditional) independence tests. VarLiNGAM and Dynotears are run with their default parameters. For CASTOR, CD-NOD, and J-PCMCI$^+$, the default parameters for the linear setting are adopted. For NegControl, the true number of nodes and edges for each dataset are provided as inputs.

 \paragraph{Results} 
 The results are summarized in Table~\ref{tab:results_exp}. First, we analyze the Mean Error Rate (MER), which remains low for most methods. For instance, our method, RCBNB-MB, achieves a MER of 1.27\% in the \textit{2 regimes} scenario, meaning that, on average, only about 8 out of 600 timestamps are misclassified. In the \textit{3 regimes} scenario, this rate further decreases to 0.81\%, demonstrating the method's reliability in correctly assigning timestamps to the appropriate regimes. In contrast, some methods, such as CASTOR, exhibit significantly higher MER values, exceeding 20\%, indicating difficulties in detecting regime changes.

\begin{table}[t]
    \centering
    \scriptsize
    \caption{{\bf Results on generated data.} The Mean Error Rate (MER), the mean F1-score, and the mean normalized Structural Hamming Distance (nSHD) across two scenarios: \textit{2 regimes} and \textit{3 regimes}. The F1-score is evaluated under three conditions: $Total$ (all edges considered), $Lagged$ (only lagged edges considered), and $Instant$ (only instantaneous edges considered). Reported values are means over  50 repetitions, with  time series of  length  600.}
    \label{tab:results_exp}
    \begin{tabular}{l|c|ccc|c|c|ccc|c}
        \hline\hline 
        & \multicolumn{5}{c|}{\textit{2 regimes}} & \multicolumn{5}{c}{\textit{3 regimes}}  \\ 
        & MER & $Total$ & $Lagged$ & $Instant$ & nSHD & MER & $Total$ & $Lagged$ & $Instant$ & nSHD\\ \hline\hline    
        {\bf RCBNB-MB} & 1.27$\%$ & \textbf{0.73}  & \textbf{0.77} & \textbf{0.67} & \textbf{0.52} & 0.81$\%$ & \textbf{0.66} & \textbf{0.71} & \textbf{0.58} & \textbf{0.55} \\
        RCBNB-PA & 1.27$\%$ & 0.69 & 0.73 & 0.62 & 0.54 & 2.93$\%$ & 0.65 & 0.69 & 0.57 & 0.57 \\ \hline
        RPCMCI-MB  & 0.24$\%$ & 0.57  & 0.65 & $\times$ & 1.03 & 0.39$\%$ & 0.54 & 0.62 & $\times$ & 1.07 \\  
        RPCMCI-PA \cite{saggioro2020reconstructing} & 1.26$\%$ & 0.57 & 0.66 & $\times$ & 1.00 & 0.35$\%$ & 0.55 & 0.63 & $\times$ & 1.04 \\ \hline
        RPCMCI$^+$-MB & 1.27$\%$ & 0.64  & 0.70 & 0.46 & 0.60 & 2.95$\%$ & 0.55 & 0.62 & 0.39 & 0.56 \\ 
        RPCMCI$^+$-PA & 10.23$\%$ & 0.55 & 0.61 & 0.36 & 0.66 & 7.79$\%$ & 0.46 & 0.52 & 0.28 & 0.73 \\ \hline
        RVarLiNGAM-MB & 0.28$\%$ & 0.43 & 0.56 & 0.09 & 0.93 & 0.37$\%$ & 0.42 & 0.55 & 0.08 & 0.93 \\
        RVarLiNGAM-PA & 0.48$\%$ & 0.43 & 0.53 & 0.17 & 0.96 & 1.78$\%$ & 0.41 & 0.51 & 0.16 & 0.96 \\ \hline
        RDynotears-MB & 16.91$\%$ & 0.19 & 0.23 & 0.07 & 1.01 & 16.25$\%$ & 0.19 & 0.23 & 0.09 & 1.02 \\
        RDynotears-PA & 15.58$\%$ & 0.18 & 0.23 & 0.07 & 0.99 & 15.03$\%$ & 0.21 & 0.25 & 0.11 & 1.02  \\
        \hline
        NegControl \cite{2025_Petersen} & $\times$ & 0.29 & 0.32 & 0.22 & 1.36 & $\times$ & 0.28 & 0.32 & 0.21 & 1.37 \\
        CD-NOD \cite{huang2020causal}& $\times$ & 0.34 & 0.40 & 0.22 & 0.87 & $\times$ & 0.31 & 0.35 & 0.21 & 0.86 \\
        J-PCMCI$^+$ \cite{gunther2023causal} & $\times$ & 0.41 & 0.48 & 0.19 & 1.13 & $\times$ & 0.37 & 0.44 & 0.15 & 1.27 \\
        CASTOR \cite{2025_RahmaniF} & 20.07$\%$ & 0.12 & 0.15 & 0.05 & 0.99 & 14.17$\%$ & 0.11 & 0.13 & 0.06 & 1.00 \\
        \hline\hline 
    \end{tabular}
\end{table}

Regarding the F1-score, RCBNB-MB achieves the highest overall performance, reaching 0.73 in the \textit{2 regimes} scenario and 0.66 in the \textit{3 regimes} scenario. This confirms its effectiveness in recovering causal structures under different conditions. The other variants of RCBNB also perform well, particularly RCBNB-PA, which maintains competitive scores. PCMCI-based methods show moderate performance, with RPCMCI-MB achieving scores around $0.64$. 
The difference between the two PCMCI-based methods is significantly smaller than the difference between the two PCMCI$^+$-based methods. This can be attributed to the fact that, when the inferred graph does not include instantaneous relations, the estimated set of parents and the estimated Markov blanket are more similar to each other than the corresponding true sets of parents and Markov blanket in the ground-truth graph, which does contain instantaneous relations.

The nSHD results further confirm this trend. RCBNB-MB achieves the lowest nSHD in both scenarios, with values of 0.52 and 0.55, closely followed by RCBNB-PA. This indicates that both variants recover graph structures closest to the ground truth. In contrast, most competing methods obtain nSHD values close to or above 1, reflecting larger structural discrepancies. 

VarLiNGAM-based methods exhibit performance comparable to J-PCMCI$^+$, while CD-NOD performs 
slightly better than the random baseline NegControl. In contrast, CASTOR and RDynotears-based methods fall below random baseline performance, highlighting their limitations in reconstructing causal graphs accurately.

Additionally, methods using the Markov blanket outperform in general those using only parents. This can be attributed to the fact that the Markov blanket offers a more robust set of variables for prediction, as it includes not only parents but also children and spouses. This representation provides MB-based methods with sufficient information to improve predictive accuracy and reduce sign errors, thereby enhancing the reliability of the subsequent causal discovery step.

The variances of the F1-scores are reported in Appendix~\ref{appendix:var_f1_600}. They are generally small, around 0.01, except for RCBNB-PA, where they range between 0.02 and 0.04, and RPCMCI$^+$, which exhibits higher variance, reaching up to 0.09 for PA and 0.05 for MB. These results suggest that while most methods yield stable performance, some, particularly RPCMCI$^+$, show more variability across repetitions.

\subsection{Real data}
\label{subsec:real_data}
\paragraph{Data description} For the real-world experiment, we use eight time series from an IT monitoring system provided by EasyVista\footnote{\url{https://www.easyvista.com/fr/produit/supervision-it/}}, sampled at one-minute intervals, as introduced in~\cite{pmlr-v206-assaad23a}. These time series describe different system activities and are detailed in Appendix~\ref{app:realdata}. According to a system expert~\cite{pmlr-v206-assaad23a}, an anomaly occurs between timesteps 46,683 and 46,783. To construct our dataset, we include an additional 1,000 timesteps before and after this interval, resulting in a total of 2,100 timesteps for analysis.

In \cite{pmlr-v206-assaad23a}, only a summary representation of the window causal graph for the normal regime was provided by the system expert. Therefore, to assess the performance of our method in this scenario, we first convert the obtained window causal graph into a summary graph~\cite{assaad2022discovery} and then compute the F1-score. Furthermore, since we have access only to the ground truth causal structure of the normal regime, our evaluation is limited to the inferred causal graphs within the regimes that overlap with the true normal regime.

\paragraph{Method configuration} Based on the method's performance in Section~\ref{subsec:simulated_data}, RCBNB-MB stands out, so we focus on this method here. Due to the increased number of variables, we set the minimum regime size to $N_{\ell}=24$ ($(2\tau_{\max}+1)d$) with $\tau_{\max}=1$.  Additionally, we set the number of initialization points to $N_i=200$, the number of regimes to $r=2 \text{ and } 3$, and allow up to $N_c=4$ transitions per regime. Other parameters remain the same as those in Section~\ref{subsec:simulated_data}.

\paragraph{Results} In this experiment, the final value of the objective function, corresponding to Equation~\ref{eq:empirical_risk_MB}, is 3.56 when assuming two regimes and 3.02 when assuming three regimes.
Figure~\ref{fig:IT_regime_assignment} illustrates the detected change points and the corresponding regimes identified by RCBNB-MB when provided with prior knowledge of either two or three regimes. According to the system expert, the ground truth consists of only two regimes: normal and abnormal. The figure clearly demonstrates that RCBNB-MB effectively detects the regimes in both cases. Even when the number of regimes is set to three, the algorithm accurately identifies the normal and abnormal regimes, with only minor deviations. Additionally, it detects a third regime at the boundary between the two true regimes, suggesting a possible transitional phase between them.
Table~\ref{tab:IT_f1} presents the F1-scores for detecting a summary representation of the window causal graph in the normal regime. Compared to previous studies, such as \cite{aïtbachir2023case}, where classical causal discovery methods were used to infer the graph, our method demonstrates superior performance.

% \begin{figure}
%     \centering
%     \includegraphics[width=\textwidth]{Images/res_regimes.png}
%     \caption{The figures above illustrate the regime assignment results obtained using the RCBNB-MB under different assumptions. The top figure corresponds to the case where three regimes are assumed: blue represents the normal regime, rose indicates the first abnormal regime, and green denotes the second abnormal regime. The middle figure shows the regime assignment when assuming only two regimes: blue for the normal regime and rose for the abnormal regime. The bottom figure presents the temporal variation of eight metrics within the IT monitoring system, aligned with the two regime assignments described above. On the horizontal axis, the circle marker ($\circ$) and the cross marker ($\times$) indicate the start and end points of the anomaly, as identified by experts, occurring at timesteps 46,683 and 46,783, respectively.}
%     \label{fig:IT_regime_assignment}
% \end{figure}

\begin{figure}[t]
    \centering
    \includegraphics[width=0.7\textwidth]{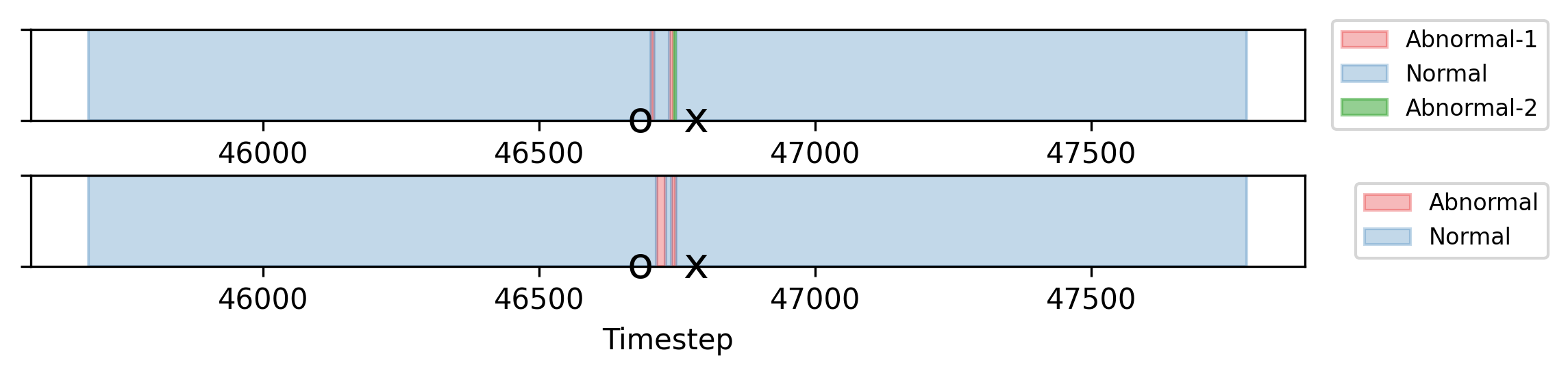}
    \caption{RCBNB-MB assignments assuming 3 regimes (top) and 2 regimes (bottom). Blue denotes normal regime, while rose and green denote abnormal regimes. The circle ($\circ$) and cross ($\times$) mark the expert identified anomaly start and end at timestamps 46,683 and 46,783, respectively.}
    \label{fig:IT_regime_assignment}
\end{figure}

% \begin{figure}
%     \centering
%     \includegraphics[width=\textwidth]{Images/SCG_IT_monitoring.pdf}
%     \caption{(a) presents the expert-provided SCG for the eight time series in the IT monitoring system under the normal regime. (b) and (c) illustrate the reconstructed SCG of RegimeMB(CBNB-w) for the blue (normal) and red (abnormal) regimes, respectively, as shown in Figure~\ref{fig:IT_regime_seperation}. }
%     \label{fig:causal_graph_IT}
% \end{figure}

\begin{table}[!t]
    \centering
    \label{tab:F1}
    %  \resizebox{9.5cm}{!}{
    \caption{Results on real data. The F1-score for each regime is computed for our method RCBNB-MB, with 2 and 3 regimes. }
    \begin{tabular}{c|c|c}
        \hline\hline 
        &loss & F1-score in normal regimes\\ \hline
        2 regimes & 3.56 & 0.53 \\  
        3 regimes & 3.02& 0.53 \\
        \hline\hline 
    \end{tabular}
    % }
        \label{tab:IT_f1}
\end{table}

%% file: Conclusion.tex
\section{Conclusion and Perspective}
\label{sec:conclu}

Understanding causal relationships in complex dynamic systems is important across many fields. Growing data availability creates opportunities but also challenges, particularly for heterogeneous time series with multiple regimes and changing causal mechanisms. We propose a method that identifies these regimes and recovers their causal structures by treating regime assignment as a prediction task based on Markov blankets. The use of Markov blankets in this setting is supported by both theoretical and experimental arguments. Lastly, we demonstrated the effectiveness of our approach in uncovering regime-dependent causal structures through an empirical validation on synthetic and real-world IT monitoring data.

%To address this challenge, we proposed a novel method which can identify different regimes and accurately recover causal relationships within the regime. This method views the regime assignment problem as a prediction task and leverages Markov blankets for this prediction. The use of Markov blankets in this setting is supported by both theoretical and experimental arguments. Lastly, we demonstrated the effectiveness of our approach in uncovering regime-dependent causal structures through an empirical validation on synthetic and real-world IT monitoring data.
%We have furthermore provided theoretical support for using showing that the reconstructed window causal graphs within each regime converge to the true graph and that the assignment of timestamps converges to the ground truth. Experiments on synthetic data demonstrate the method’s effectiveness in correctly assigning timestamps to the appropriate regimes and discovering the correct window causal graph within each regime.

There are several interesting aspects to explore in future work. In our experiments, we only considered linear relationships among variables and time series containing two or three distinct regimes. More complex scenarios involving non-linear relationships and more than three distinct regimes need to be tested. Additionally, timestamps at the boundaries of two distinct regimes are difficult to assign correctly due to the loss of information about the Markov blanket. Further research is needed to improve the assignment of these boundary timestamps. Lastly, we would also like to explore its potential for root cause analysis and anomaly detection.

%In terms of applications, it would be valuable to evaluate our proposed method on additional real-world datasets, such as climate data or epidemiological data. Furthermore, exploring its potential for root cause analysis and anomaly detection would be particularly insightful.
%Regarding applications, it would be interesting to test the performance of our proposed method in root cause analysis. The data used in this task often exhibit heterogeneity, typically including two regimes: normal and abnormal. Determining the transition between these regimes is also challenging. By applying RCBNB-MB, \textcolor{red}{we can theoretically identify  root causes and pinpoint the timestamp when the system entered an anomalous state. Root causes are nodes whose structures (incoming and outgoing edges) change between the window causal graphs of the normal and abnormal regimes. The onset of the abnormal event is indicated by the regime change point.}

%% file: Appendix.tex
\section{Proof of Theorem~\ref{th:regime-criterion}}
\label{appendix:th}

% For certain families of functions, which typically are universal approximators, and under certain conditions which depend on the family considered, see for example \cite{gyorfi2002distribution}, the estimate $\widehat{h}^{j,k}_{\mathbf{T}}$ converges to the ''best'' function $h^{j,k}_{*}$ in the following sense:
% %
% \begin{equation} \label{eq:univ-approx}
%     \lim_{k(\mathbf{T}) \rightarrow +\infty} %\mathop{\mathbb{E}} \left[ 
%     \int_{\mathbf{MB}(x_k^j;\mathcal{G}_w^k)} \left( \widehat{h}^{j,k}_{\mathbf{T}}(\mathbf{MB}(x_k^j;\mathcal{G}_w^k)) - h^{j,k}_{*}(\mathbf{MB}(x_k^j;\mathcal{G}_w^k)) \right)^2 dP(\mathbf{MB}(x_k^j;\mathcal{G}_w^k)) %\right]
%     = 0,
% \end{equation}
% %
% where $k(\mathbf{T}) = \sum_{t \in \mathbf{T}} \mathds{1}_{C_t=k}$ and the integral is taken over all possible values of the Markov blanket of all variables $X_{t,t \in \mathbf{T}}^j$ s.t. $C_t=k$.

% The above development leads to the following theorem which provides a theoretical criterion for deciding the regime for a given variable at a particular time instant.
%\primetheorem*
\setcounter{theorem}{0}
\begin{theorem}
    Suppose that \(C_t=k\), that the function \(h^{j,k}_{*}(\mathbf{MB}(x_t^j;\mathcal{G}_w^k))\) approximates \(x_t^j\) within a bounded error, i.e., \(\exists M> 0, \forall x_t^j, |x_t^j - h^{j,k}_{*}(\mathbf{MB}(x_t^j;\mathcal{G}_w^k))| < M\), and that the estimate $\widehat{h}^{j,k}_{\mathbf{T}}$ converges to the ''best'' function $h^{j,k}_{*}$ in the following sense:
 \begin{equation} \label{eq:univ-approx}
     \lim_{k(\mathbf{T}) \rightarrow +\infty} %\mathop{\mathbb{E}} \left[ 
     \int_{\mathbf{S}} \left( \widehat{h}^{j,k}_{\mathbf{T}}(\mathbf{S}) - h^{j,k}_{*}(\mathbf{S}) \right)^2 dP(\mathbf{S}) %\right]
     = 0,
 \end{equation} where $\mathbf{S}$ is any subset of $\mathbf{X}$. Then, for any other regime \(\ell \neq k\):
    \[
    \lim_{k(\mathbf{T}) \to +\infty} \mathop{\mathbb{E}}\left[ (X^j_t - \widehat{h}^{j,k}_{\mathbf{T}}(\mathbf{MB}(X_t^{j};\mathcal{G}_w^k)))^2\right] \le \lim_{k(\mathbf{T}) \to +\infty} \mathop{\mathbb{E}}\left[ (X^j_t - \widehat{h}^{j,\ell}_{\mathbf{T}}(\mathbf{MB}(X_t^{j};\mathcal{G}_w^{\ell})))^2\right].
    \]
\end{theorem}

%\begin{theorem}\label{th:regime-criterion}
%     Let us assume that $C_t=k$ and that the function $h^{j,k}_{*}(\mathbf{MB}(x_t^j;\mathcal{G}_w^k))$ is never too far from $x_t^j$, that is: $\exists M> 0, \forall x_t^j, |x_t^j - h^{j,k}_{*}(\mathbf{MB}(x_t^j;\mathcal{G}_w^k))| < M$. Then $\forall 1 \le \ell \ne k \le r, \forall \epsilon > 0, \exists \mathbf{T}_0$ s. t. $\forall k(\mathbf{T}) \ge k(\mathbf{T}_0)$,
%     \[
%     \mathop{\mathbb{E}}\left[ (X^j_t - \widehat{h}^{j,k}_{\mathbf{T}}(\mathbf{MB}(X_t^{j};\mathcal{G}_w^k)))^2\right] < \mathop{\mathbb{E}}\left[ (X^j_t - \widehat{h}^{j,\ell}_{\mathbf{T}}(\mathbf{MB}(X_t^{j};\mathcal{G}_w^{\ell})))^2\right] + \epsilon.
%     \]
%     In other words, $\forall 1 \le \ell \ne k \le r$:
%     \[
%     \lim_{k(\mathbf{T}) \rightarrow +\infty} \mathop{\mathbb{E}}\left[ (X^j_t - \widehat{h}^{j,k}_{\mathbf{T}}(\mathbf{MB}(X_t^{j};\mathcal{G}_w^k)))^2\right] \le \lim_{k(\mathbf{T}) \rightarrow \infty} \mathop{\mathbb{E}}\left[ (X^j_t - \widehat{h}^{j,\ell}_{\mathbf{T}}(\mathbf{MB}(X_t^{j};\mathcal{G}_w^{\ell})))^2\right].
%     \]
% \end{theorem}
% %
 \begin{proof}
     Let $\epsilon > 0$. Equation~\ref{eq:univ-approx} implies that $\forall \epsilon' > 0, \exists \mathbf{T}_0$ s.t. $\forall k(\mathbf{T}) \ge k(\mathbf{T}_0), \forall \mathbf{MB}(x_k^j;\mathcal{G}_w^k)$,
     \[
     \left( \widehat{h}^{j,k}_{\mathbf{T}}(\mathbf{MB}(x_k^j;\mathcal{G}_w^k)) - h^{j,k}_{*}(\mathbf{MB}(x_k^j;\mathcal{G}_w^k)) \right)^2 < \epsilon'^2,
     \]
     which implies, by developing $(x_t^j - \widehat{h}^{j,k}_{\mathbf{T}}(\mathbf{MB}(X_t^{j};\mathcal{G}_w^k)))^2$ as $(x_t^j - h^{j,k}_{*}(\mathbf{MB}(x_k^j;\mathcal{G}_w^k)) + h^{j,k}_{*}(\mathbf{MB}(x_k^j;\mathcal{G}_w^k)) - \widehat{h}^{j,k}_{\mathbf{T}}(\mathbf{MB}(X_t^{j};\mathcal{G}_w^k)))^2$, that $\forall x_t^j$:
     \[
     (x_t^j - \widehat{h}^{j,k}_{\mathbf{T}}(\mathbf{MB}(x_t^{j};\mathcal{G}_w^k)))^2 < (x_t^j - h^{j,k}_{*}(\mathbf{MB}(x_t^{j};\mathcal{G}_w^k)))^2 + \epsilon'^2 + 2 M\epsilon'.
     \]
     Choosing $\epsilon'$ such that $\epsilon'^2 + 2 M\epsilon' < \epsilon$ and taking the expectation leads to:
     \begin{equation}\label{eq:exp1}
     \mathop{\mathbb{E}}\left[ (X^j_t - \widehat{h}^{j,k}_{\mathbf{T}}(\mathbf{MB}(X_t^{j};\mathcal{G}_w^k)))^2\right] < \mathop{\mathbb{E}}\left[ (X^j_t - h^{j,k}_{*}(\mathbf{MB}(X_t^{j};\mathcal{G}_w^{k})))^2\right] + \epsilon.
     \end{equation}
     As $C_t=k$, one has (Equation~\ref{eq:minimal_b}):
     \begin{equation}\label{eq:exp2}
     \mathop{\mathbb{E}}\left[ (X^j_t - h^{j,k}_{*}(\mathbf{MB}(X_t^{j};\mathcal{G}_w^k)))^2\right] \le \mathop{\mathbb{E}}\left[ (X^j_t - \widehat{h}^{j,\ell}_{\mathbf{T}}(\mathbf{MB}(X_t^{j};\mathcal{G}_w^{\ell})))^2\right].
     \end{equation}
     Combining the inequalities~\ref{eq:exp1} and \ref{eq:exp2} concludes the proof.
 \end{proof}
% %

% Provided that the boundedness assumption holds for each variable, Theorem~\ref{th:regime-criterion} implies that $\forall \epsilon > 0, \exists \mathbf{T}_0$ s. t. $\forall k(\mathbf{T}) \ge k(\mathbf{T}_0)$,
%     \begin{equation}\label{eq:sumexp}
%     \sum_{j=1}^d \mathop{\mathbb{E}}\left[ (X^j_t - \widehat{h}^{j,k}_{\mathbf{T}}(\mathbf{MB}(X_t^{j};\mathcal{G}_w^k)))^2\right] < \sum_{j=1}^d \mathop{\mathbb{E}}\left[ (X^j_t - \widehat{h}^{j,\ell}_{\mathbf{T}}(\mathbf{MB}(X_t^{j};\mathcal{G}_w^{\ell})))^2\right] + \epsilon.
%     \end{equation}
% It also implies that, provided that all $k(\mathbf{T}), \, 1 \le k \le r$, are sufficiently large, if, for any $j$, one has that for all $1 \le \ell \ne k \le r$:
% %
% \begin{equation}\label{crit1}
% \mathop{\mathbb{E}}\left[ (X^j_t - \widehat{h}^{j,k}_{\mathbf{T}}(\mathbf{MB}(X_t^{j};\mathcal{G}_w^k)))^2\right] < \mathop{\mathbb{E}}\left[ (X^j_t - \widehat{h}^{j,\ell}_{\mathbf{T}}(\mathbf{MB}(X_t^{j};\mathcal{G}_w^{\ell})))^2\right],
% \end{equation}
% %
% %\textcolor{red}{pour tout j? somme sur j?}
% then $t$ should be assigned to regime $k$. There are three cases which can lead to an equality between the two terms in Inequality~\ref{crit1}: the two Markov blankets are equal, one is a superset of the other, there are other sets of variables, in addition to the Markov blanket, minimizing Equation~\ref{eq:minimal_b}. The following assumption allows one to differentiate between these cases. %To rule out these cases, we rely on the following assumption.

\clearpage
\section{Pseudo-code of RCBNB-MB}
\label{appendix:pseudo_code}
Algorithm~\ref{alg:regime_CBNBW} presents the pseudo-code of RCBNB-MB, which iteratively partitions the time series into regimes and infers regime-specific window causal graphs using the CBNB method. The algorithm initializes multiple segmentations and alternates between estimating causal structures and updating regime assignments by minimizing the empirical risk function (Equation~\eqref{eq:empirical_risk_MB}) subject to Constraints 1 and 2 in Section~\ref{sec:method}. Within each regime, it determines the Markov Blanket of each variable to enhance segmentation accuracy. After multiple iterations, the segmentation and window causal graphs corresponding to the lowest empirical risk are selected.

\begin{algorithm}[!ht]
\caption{RCBNB-MB}
\label{alg:regime_CBNBW}
\begin{algorithmic}[1]
\STATE \textbf{Input:} Heterogeneous time series data $\{\mathbf{x}_{t}\}_{t \in \mathbf{T}}$, number of initialization points $N_i$, maximal number of optimization iterations $N_o$, number of regimes $r$, minimum size per regime $N_{\ell}$, maximum transition points per regime $N_c$, the family of $\mathbf{H}$.

\FOR{$a = 0$ \TO $N_i$}
    \STATE Initialize $\mathbf{C}$ by $\mathbf{C}^{[0]}$ ensuring each regime index lasts at least $(2\tau_{\max}+1)\times d$ consecutive timestamps.
    \FOR{$\text{ite} = 0$ \TO $N_o$}
        \STATE \textbf{Step 1: Estimate the window causal graph and estimating functions for each regime}
        \STATE Determine $\{\Upsilon^k\}_{k \in \{1, \dots, r\}}$ based on $\mathbf{C}^{[\text{ite}]}$.
        \FOR{$k = 1$ \TO $r$}
            \STATE Use CBNB on $\{\mathbf{x}_{t}\}_{t \in \Upsilon^k}$ to estimate $\mathcal{G}_w^k$.
            \STATE For each dimension $j \in \{1, \dots, d\}$, deduce $\mathbf{MB}({X}^j_{t}; \mathcal{G}_w^k)$, the smallest set of variables that renders ${X}^j_{t}$ conditionally independent from all others within $\mathcal{G}_w^k$.
            \STATE Update $\mathbf{H}^{[\text{ite}]}$ for each dimension $j \in \{1, \dots, d\}$ within regime $k$ using $\{\mathbf{x}_{t}\}_{t \in \Upsilon_k}$ and the corresponding $\mathbf{MB}({X}^j_{t}; \mathcal{G}_w^k)$.
        \ENDFOR
        \STATE \textbf{Step 2: Update regime assignments}
        \STATE Update $\mathbf{C}^{[\text{ite+1}]}$ by solving the minimization problem in Equation~\eqref{eq:empirical_risk_MB} subject to Constraints 1 and 2 in Section~\ref{sec:method}. Save the resulting value of Equation~\eqref{eq:empirical_risk_MB} as $\mathbf{L}^{[\text{ite}]}$.
    \STATE Save $List\mathbf{L}[a] = \mathbf{L}^{[N_o]}$, $List\mathbf{C}[a] = \mathbf{C}^{[N_o+1]}$, and $List\mathcal{G}[a] = \{\mathcal{G}_w^k\}_{k \in \{1, \dots, r\}}^{[N_o]}$.
    \ENDFOR
\ENDFOR
\STATE $Index = \arg\min_{a} List\mathbf{L}[a]$.
\STATE $\mathbf{C} = List\mathbf{C}[Index]$, $\{\mathcal{G}_w^k\}_{k \in \{1, \dots, r\}} = List\mathcal{G}[Index]$.
\STATE \textbf{Output:} Regime assignments $\mathbf{C}$ and window causal graphs $\{\mathcal{G}_w^k\}_{k \in \{1, \dots, r\}}$.

\end{algorithmic}
\end{algorithm}

\clearpage
\section{Supplementary Experiments}
\subsection{Parameter Robustness}
\label{appendix:para_robustness}
In this section, we investigate the sensitivity of RCBNB-MB to three key hyperparameters: the number of initialization points $N_i$, the maximum number of optimization iterations $N_o$, and the minimum size per regime $N_{\ell}$. We use the simulated datasets from Section~\ref{subsec:simulated_data}, which consist of 50 datasets, each containing 6 variables with time series of length 600. We set $N_i = 50$, $N_o = 20$, and $N_{\ell} = 18$ as the default configuration, and vary each hyperparameter independently while keeping the others fixed.

\paragraph{Effect of $N_i$.} We vary $N_i \in \{10, 30, 50, 70, 90\}$ while fixing $N_o$ and $N_{\ell}$. Figures~\ref{fig:N_i_f1} and~\ref{fig:N_i_mer} show the mean F1-score (with standard deviation) and the MER as functions of $N_i$, respectively. The mean F1-score increases gradually with $N_i$, while the standard deviation remains approximately constant. Correspondingly, MER exhibits a declining trend as $N_i$ grows. This behavior is expected, as a larger number of initialization points increases the probability that the optimization procedure finds an assignment close to the true regime segmentation, thereby improving causal graph reconstruction within each regime. Importantly, the performance remains relatively stable across the evaluated range of $N_i$.

\paragraph{Effect of $N_o$.} We vary $N_o \in \{5, 10, 20, 30, 40\}$ while fixing $N_i$ and $N_{\ell}$. Figures~\ref{fig:N_o_f1} and~\ref{fig:N_o_mer} report the corresponding F1-score and MER. The F1-score shows a modest increasing trend with $N_o$, while the standard deviation remains stable, and MER generally decreases as $N_o$ grows, despite a transient peak at $N_o=10$. This is consistent with the intuition that allowing more optimization iterations per run increases the likelihood of converging to a more accurate regime assignment and causal structure.

\paragraph{Effect of $N_{\ell}$.} We vary $N_{\ell} \in \{6, 12, 18, 24, 30\}$ while fixing $N_i$ and $N_o$. Figures~\ref{fig:N_l_f1} and~\ref{fig:N_l_mer} show the corresponding results. Both the F1-score and MER remain flat across the evaluated range, confirming that RCBNB-MB is largely insensitive to this parameter. This is consistent with its role as a safeguard, as $N_{\ell}$ enforces a minimum number of time steps per regime to ensure that the regression and causal discovery steps within the optimization are statistically well-posed. When the optimization proceeds normally, this constraint is rarely binding, and its value has little effect on the final outcome. Its primary function is to prevent degenerate cases in which no time steps are assigned to a given regime.

\begin{figure}[!h]
\centering

\begin{subfigure}{0.45\textwidth}
\centering
\includegraphics[width=\linewidth]{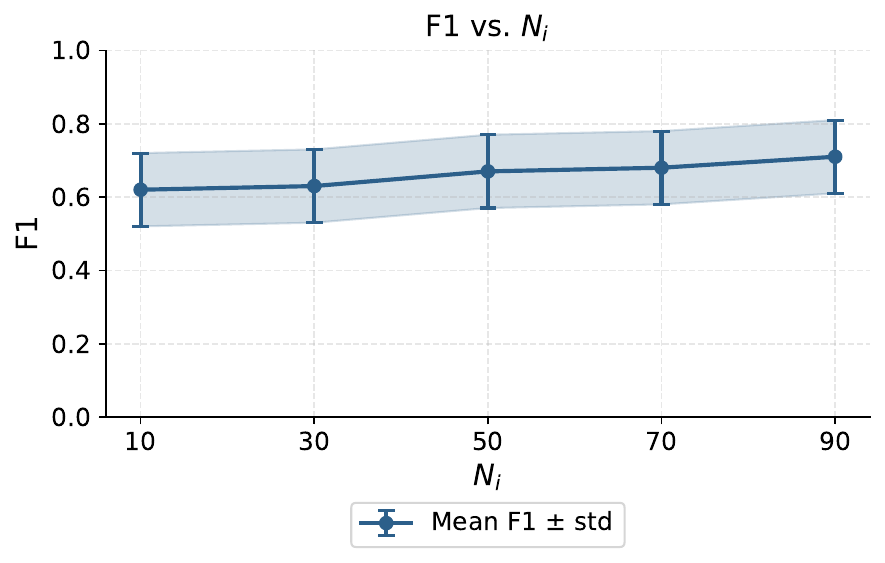}
\caption{F1-score vs.\ $N_i$.}
\label{fig:N_i_f1}
\end{subfigure}
\hfill
\begin{subfigure}{0.45\textwidth}
\centering
\includegraphics[width=\linewidth]{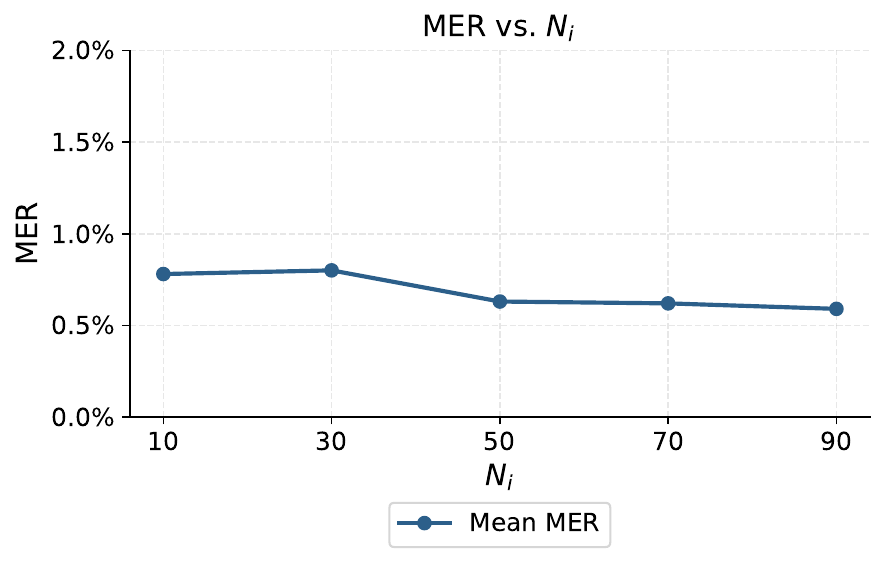}
\caption{MER vs.\ $N_i$.}
\label{fig:N_i_mer}
\end{subfigure}

\begin{subfigure}{0.45\textwidth}
\centering
\includegraphics[width=\linewidth]{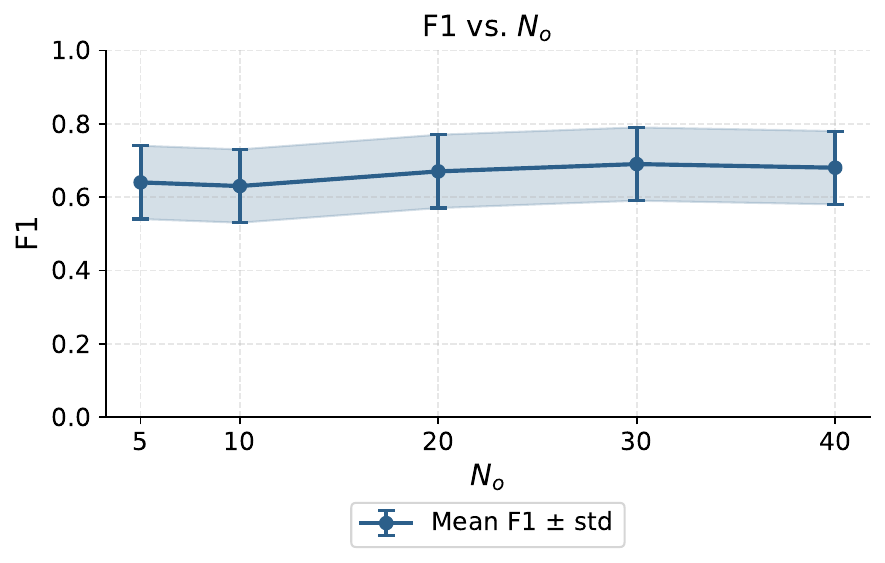}
\caption{F1-score vs.\ $N_o$.}
\label{fig:N_o_f1}
\end{subfigure}
\hfill
\begin{subfigure}{0.45\textwidth}
\centering
\includegraphics[width=\linewidth]{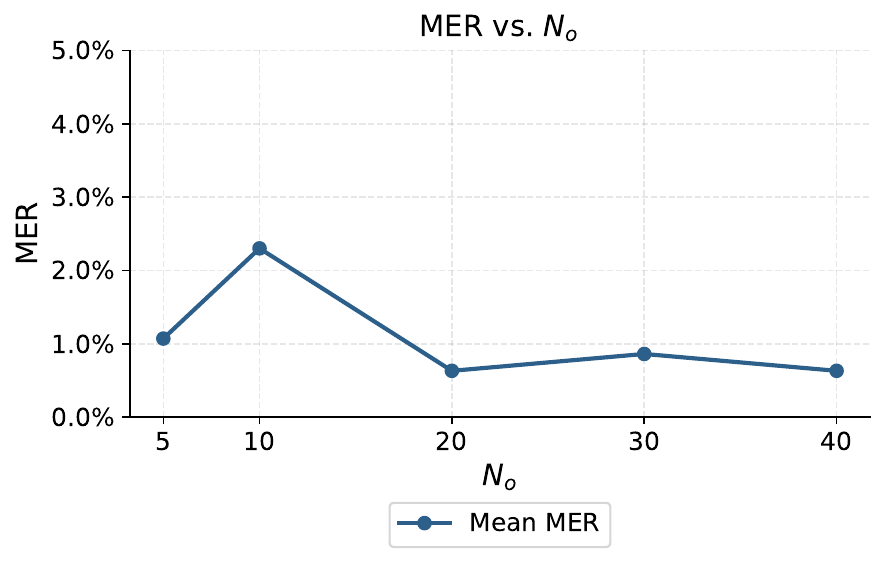}
\caption{MER vs.\ $N_{o}$.}
\label{fig:N_o_mer}
\end{subfigure}

\begin{subfigure}{0.45\textwidth}
\centering
\includegraphics[width=\linewidth]{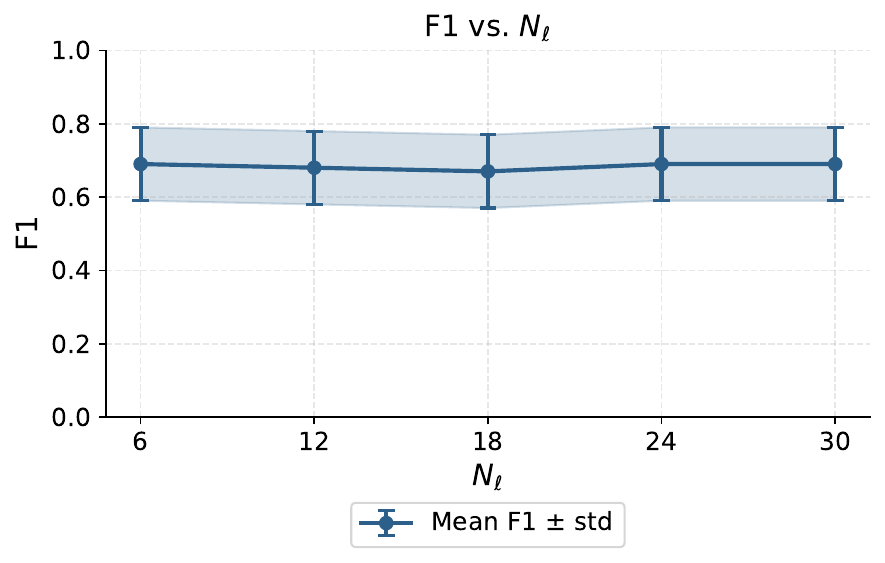}
\caption{F1-score vs.\ $N_{\ell}$.}
\label{fig:N_l_f1}
\end{subfigure}
\hfill
\begin{subfigure}{0.45\textwidth}
\centering
\includegraphics[width=\linewidth]{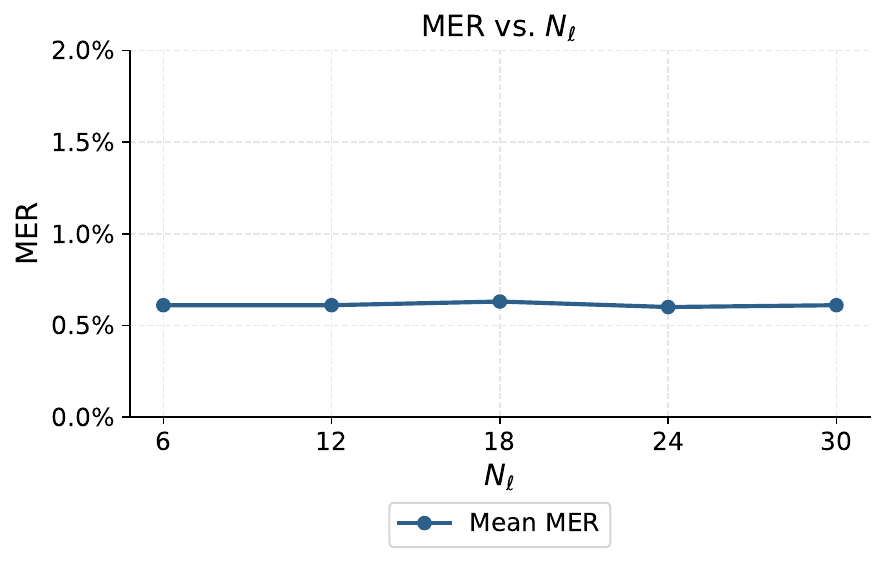}
\caption{MER vs.\ $N_{\ell}$.}
\label{fig:N_l_mer}
\end{subfigure}

\caption{Sensitivity analysis of RCBNB-MB with respect to three hyperparameters: the number of initialization points $N_i$ (top row), the maximum number of optimization iterations $N_o$ (middle row), and the minimum size per regime $N_\ell$ (bottom row). The left and right columns report the mean F1-score with standard deviation and the mean MER, respectively, evaluated over 50 simulated datasets with a time series length of 600.}
\label{fig:hyper_parameter_varying}

\end{figure}

\clearpage
\subsection{Simulated Data with a Length of 1,200}
\label{appendix:simu_1200}

The results of each method are presented in Table~\ref{tab:results_exp_1200}. We first analyze the Mean Error Rate (MER), which remains low for most methods. Compared to the case where the length of the time series is 600, most methods exhibit a lower MER. Notably, RCBNB-MB achieves a MER of 0.14\% in the \textit{2 regimes} scenario and 0.64\% in the \textit{3 regimes} scenario. In contrast, certain methods, such as CASTOR, display significantly higher MER values, exceeding 23\%, indicating difficulties in detecting regime changes.  

Regarding the F1-score, RCBNB-MB maintains strong performance, achieving 0.71 in both the \textit{2 regimes} and \textit{3 regimes} scenarios. RCBNB-PA demonstrates comparable performance and even outperforms RCBNB-MB in the \textit{2 regimes} scenario with an F1-score of 0.72. PCMCI-based methods show moderate performance, with RPCMCI-PA achieving scores around 0.58. 
The performance gap between the PCMCI-based methods is notably smaller than that observed between the PCMCI$^+$-based methods. CASTOR and RDynotears-based methods exhibit the lowest F1-scores, highlighting their limitations in reconstructing causal graphs accurately. Similar to the case where the time series length is 600, methods incorporating the Markov Blanket perform at least as well as those relying solely on parents.

The nSHD results show a similar trend. RCBNB-MB achieves the lowest nSHD in both scenarios, with 0.48 for \textit{2 regimes} and 0.55 for \textit{3 regimes}, followed closely by RCBNB-PA. This indicates that RCBNB-based methods recover graph structures closest to the ground truth. RPCMCI$^+$-MB also obtains competitive nSHD values, whereas most competing methods have values close to or above 1, reflecting larger structural discrepancies. 

Table~\ref{tab:F1_var_1200} reports the variance of the F1-scores and the normalized Structural Hamming Distance (nSHD). In general, the variances are small, around 0.01. RPCMCI$^+$ is an exception, which exhibits a higher variance, reaching up to 0.09. These findings suggest that while most methods demonstrate stable performance, some, particularly RPCMCI$^+$, exhibit greater variability across repetitions.

\begin{table}[!ht]
    \centering
    \caption{The Mean Error Rate (MER), the F1-score, and the normalized Structural Hamming Distance (nSHD) across two scenarios: \textit{2 regimes} and \textit{3 regimes}. The F1-score is evaluated under three conditions: $Total$ (all edges considered), $Lagged$ (only lagged edges considered), and $Instant$ (only instantaneous edges considered). The reported values are based on 50 repetitions, with each time series having a length of 1,200.}
           	\begin{subtable}[h]{1\textwidth}
            \centering 
            \scriptsize
    \begin{tabular}{c|c|ccc|c|c|ccc|c}
        \multicolumn{10}{c}{\textit{Equal Regime Sizes}} \\
        \hline\hline 
        & \multicolumn{5}{c|}{\textit{2 regimes}} & \multicolumn{5}{c}{\textit{3 regimes}}  \\ 
        & MER & $Total$ & $Lagged$ & $Instant$ & nSHD & MER & $Total$ & $Lagged$ & $Instant$ & nSHD \\ \hline\hline    
        RCBNB-MB & 0.14$\%$ & 0.71  & \textbf{0.76} & 0.61 & \textbf{0.48} & 0.64$\%$ & \textbf{0.71}  & \textbf{0.78} & 0.58 & \textbf{0.55}	 \\
        RCBNB-PA & 2.12$\%$ & \textbf{0.72 } & 0.75 & \textbf{0.66} & 0.51 & 1.05$\%$ & \textbf{0.71} & 0.75 & \textbf{0.62} & 0.57 \\ \hline
        RPCMCI-MB & 0.14$\%$ & 0.57  & 0.65 & $\times$  & 1.03 & 0.18$\%$ & 0.56 & 0.64 & $\times$ & 1.07 \\  
        RPCMCI-PA  & 0.14$\%$ & 0.58  & 0.67 & $\times$  & 0.97	 & 0.17$\%$ & 0.57 & 0.65 & $\times$ & 1.04 \\ \hline
        RPCMCI$^+$-MB & 0.15$\%$  & 0.68 & 0.75 & 0.49 & 0.58 & 1.06$\%$ & 0.67 & 0.73 & 0.48 & 0.56 \\ 
        RPCMCI$^+$-PA & 9.11$\%$  & 0.56 & 0.62 & 0.40 & 0.66 & 5.92$\%$ & 0.51 & 0.57 & 0.32 & 0.73  \\ \hline
        RVarLiNGAM-MB & 0.12$\%$ & 0.42 & 0.56 & 0.09 & 0.92 & 0.19$\%$ & 0.44 & 0.56 & 0.12 & 0.93  \\
        RVarLiNGAM-PA & 1.18$\%$ & 0.42 & 0.52 & 0.16 & 0.98 & 2.89$\%$ & 0.43 & 0.53 & 0.18 & 0.96 \\ \hline
        RDynotears-MB& 15.49$\%$ & 0.15 & 0.19 & 0.06 & 1.00 & 15.12$\%$ & 0.14 & 0.17 & 0.06 & 1.02 \\
        RDynotears-PA & 8.73$\%$ & 0.14 & 0.18 & 0.05 & 0.98 & 13.89$\%$ & 0.16 & 0.20 & 0.08 &  1.02 \\
        \hline
        NegControl & $\times$ & 0.30 & 0.33 & 0.23  & 1.34 & $\times$ & 0.28 & 0.32 & 0.21 & 1.37  \\
        CD-NOD & $\times$ & 0.38 & 0.43 & 0.23 & 0.86 & $\times$ & 0.36 & 0.41 & 0.22 & 0.86 \\
        J-PCMCI$^+$ & $\times$ & 0.43 & 0.52 & 0.20 & 1.16 & $\times$ & 0.39 & 0.47 & 0.15 & 1.27 \\
        CASTOR & 23.11$\%$ & 0.12 & 0.15 & 0.06 & 0.99 & 23.46$\%$ & 0.08 & 0.10 & 0.03 & 1.00  \\
        \hline\hline 
    \end{tabular}
    \caption{Mean over 50 repetitions.}
        \label{tab:results_exp_1200}
\end{subtable}
% \end{table}

% \begin{table}[!ht]
%     \centering
%     \caption{Variance of the F1-score across two scenarios: \textit{2 regimes} and \textit{3 regimes}. For each scenario, the F1-score is evaluated under three conditions: $Total$ (all edges considered), $Lagged$ (only lagged edges considered), and $Instant$ (only instantaneous edges considered). Additionally, only oriented edges are considered. The reported values are based on 50 repetitions. Each time series has a length of 1,200, with the disturbance term distributed as $\epsilon_{t}^{j} \sim \mathbf{U}(-0.1, 0.1)$.}
       	\begin{subtable}[h]{1\textwidth}
        \centering
        \scriptsize
    \begin{tabular}{c|ccc|c|ccc|c}
         \multicolumn{9}{c}{\textit{Equal Regime Sizes}} \\   
        \hline\hline 
        & \multicolumn{4}{c|}{\textit{2 regimes}} & \multicolumn{4}{c}{\textit{3 regimes}}  \\ 
        & $Total$ & $Lagged$ & $Instant$ & nSHD & $Total$ & $Lagged$ & $Instant$ & nSHD\\ \hline\hline 
        RCBNB-MB & 0.01  & 0.01 & 0.04 & 0.03 & 0.01  & 0.01 & 0.02 & 0.02  \\
        RCBNB-PA & 0.02 & 0.02 & 0.04 & 0.03 & 0.02 & 0.03 & 0.03 & 0.03 \\ \hline
        RPCMCI-MB & <0.01  & <0.01 & $\times$ & 0.03 & <0.01 & <0.01 & $\times$ & 0.02	  \\  
        RPCMCI-PA  & <0.01  & <0.01 & $\times$ & 0.04 & <0.01 & <0.01 & $\times$ & 0.02	\\ \hline
        RPCMCI$^+$-MB & 0.01  & 0.01 & 0.03 & 0.03 & 0.02 & 0.03 & 0.02 & 0.02	  \\ 
        RPCMCI$^+$-PA & 0.07  & 0.08 & 0.05 & 0.05 & 0.08 & 0.09 & 0.04 & 0.03 \\ \hline
        RVarLiNGAM-MB & <0.01 & <0.01 & 0.01 & 0.01 & <0.01 & <0.01 & 0.01 & 0.01 \\
        RVarLiNGAM-PA & <0.01 & 0.01 & 0.01 & 0.01 & <0.01 & <0.01 & 0.01 & 0.01 \\ \hline
        RDynotears-MB & 0.01 & 0.02 & 0.01 & <0.01 & 0.01 & 0.02 & 0.01 & 0.01 \\
        RDynotears-PA & 0.01 & 0.01 & 0.01 & <0.01 & 0.01 & 0.01 & 0.01 & 0.01 \\
        \hline 
        NegControl  & <0.01 & 0.01 & 0.01  & 0.02 & <0.01 & <0.01 & 0.01 & 0.01	 \\
        CD-NOD  & 0.01 & 0.01 & 0.02 & 0.01 & 0.01  & 0.01 & 0.01 & 0.01  \\
        J-PCMCI$^+$  & <0.01 & <0.01 & 0.01 & 0.02 & <0.01 & <0.01 & <0.01 & 0.02 \\
        CASTOR & 0.01 & 0.01 & 0.01  & <0.01 & <0.01 & 0.01 & <0.01 & <0.01 \\
        \hline\hline 
    \end{tabular}   
    \caption{Variance over 50 repetitions.}
    \label{tab:F1_var_1200}
\end{subtable}
\end{table}

\clearpage
\subsection{Variance of the F1-score and the normalized Structural Hamming Distance (nSHD)for Simulated Data with a Length of 600}
\label{appendix:var_f1_600}
Table~\ref{tab:F1_var_600} presents the variances of the F1-scores for simulated data with a length of 600. The variances are generally small, around 0.01, except for RCBNB-PA, where they range between 0.02 and 0.04, and RPCMCI$^+$, which exhibits higher variance, reaching up to 0.09 for PA and 0.05 for MB. These results suggest that while most methods yield stable performance, some, particularly RPCMCI$^+$, show more variability across repetitions.

The nSHD variances are also generally small, indicating stable structural performance across repetitions. For RCBNB-based methods, the variance remains low, ranging from 0.01 to 0.04. Most competing methods show similarly limited variability, although RPCMCI-based methods in the \textit{2 regimes} scenario exhibit slightly higher nSHD variance, reaching 0.06 for MB and 0.07 for PA. 

\begin{table}[!ht]
    \centering
    \scriptsize
    \caption{Variance of the F1-score and the normalized Structural Hamming Distance (nSHD) across two scenarios: \textit{2 regimes} and \textit{3 regimes}. For each scenario, the F1-score is evaluated under three conditions: $Total$ (all edges considered), $Lagged$ (only lagged edges considered), and $Instant$ (only instantaneous edges considered). Additionally, only oriented edges are considered. The reported values are based on 50 repetitions. Each time series has a length of 600, with the disturbance term distributed as $\epsilon_{t}^{j} \sim \mathbf{U}(-0.1, 0.1)$.}
    \label{tab:F1_var_600}
    \begin{tabular}{c|ccc|c|ccc|c}
        \multicolumn{7}{c}{\textit{Equal Regime Sizes}} \\
        \hline\hline 
        & \multicolumn{4}{c|}{\textit{2 regimes}} & \multicolumn{4}{c}{\textit{3 regimes}}  \\ 
        & $Total$ & $Lagged$ & $Instant$ & nSHD & $Total$ & $Lagged$ & $Instant$ & nSHD \\ \hline\hline    
        RCBNB-MB & 0.01  & 0.01 & 0.02 & 0.04 & 0.01 & 0.01 & 0.03 & 0.01 \\
        RCBNB-PA & 0.03 & 0.03 & 0.04 & 0.04 & 0.02 & 0.02 & 0.03 & 0.04 \\ \hline
        RPCMCI-MB & <0.01  & 0.01 & $\times$ & 0.06	 & <0.01 & <0.01 & $\times$ & 0.02 \\  
        RPCMCI-PA  & 0.01  & 0.01 & $\times$ & 0.07 & <0.01 & <0.01 & $\times$ & 0.02 \\ \hline
        RPCMCI$^+$-MB & 0.02  & 0.02 & 0.03 & 0.04 & 0.04 & 0.05 & 0.03 & 0.02 \\ 
        RPCMCI$^+$-PA & 0.07 & 0.08 & 0.05 & 0.05 & 0.07 & 0.09 & 0.03 & 0.03 \\ \hline
        RVarLiNGAM-MB & <0.01 & 0.01 & 0.01 & 0.01 & <0.01 & <0.01 & 0.01 & 0.01 \\
        RVarLiNGAM-PA & <0.01 & <0.01 & 0.02 & 0.02	& <0.01 & <0.01 & 0.01 & 0.01\\ \hline
        RDynotears-MB & 0.01 & 0.02 & 0.01 & 0.02 & 0.01 & 0.01 & 0.01 & 0.02 \\
        RDynotears-PA & 0.01 & 0.02 & 0.01 & 0.01 & 0.01 & 0.01 & 0.01 & 0.02 \\
        \hline
        NegControl & <0.01 & 0.01 & 0.01 & 0.02 & <0.01 & <0.01 & 0.01 & 0.01 \\
        CD-NOD & 0.02 & 0.03 & 0.02 & 0.02 & 0.01 & 0.01 & 0.01 & 0.01	 \\
        J-PCMCI$^+$  & 0.01 & 0.01 & 0.01 & 0.03 & <0.01 & 0.01 & 0.01 & 0.03 \\
        CASTOR & 0.01 & 0.01 & 0.01 & <0.01 & <0.01 & 0.01 & <0.01 & <0.01 \\
        \hline\hline 
    \end{tabular}
\end{table}

\clearpage
\subsection{Simulated Data with Unequal Regime Sizes}
\label{appendix:unequal_regime_size}
In this section, we evaluate the methods in a scenario with unequal regime sizes. The data generation process follows Equation~\ref{eq:linear_SCM}, using the same parameter distributions as in Section~\ref{subsec:simulated_data}. The simulated system contains six variables with linear relationships. The main difference from the previous setting is that the size of each regime is randomly determined, subject to the constraint that each regime occupies at least 10\% of the time series and remains contiguous ($i.e.$, without interruptions). The length of the time series varies between 600 and 1,200, and the number of regimes ranges from 2 to 3.

Table~\ref{tab:results_exp_l_unequal_600} reports the Mean Error Rate (MER) and the mean F1-score for the \textit{2 regimes} and \textit{3 regimes} scenarios when the time series length is 600. In terms of MER, when the number of regimes is two, RCBNB-MB achieves a MER of 0.61\%, which remains competitive with RPCMCI-MB (0.48\%), RPCMCI-PA (0.49\%), and RVarLiNGAM-MB (0.57\%). When the number of regimes increases to three, these methods continue to maintain relatively low MER values. Compared with the equal regime size scenario, MER generally increases across most methods, reflecting the additional difficulty of detecting regime changes when regime sizes vary. Nevertheless, RPCMCI-MB and RPCMCI-PA achieve slightly lower MER values when the number of regimes is two, indicating strong performance in regime assignment. In contrast, CASTOR and RDynotears-based methods still show relatively poor performance on this task.

Regarding the F1-score, RCBNB-MB achieves the best performance in both the \textit{2 regimes} and \textit{3 regimes} scenarios with total F1-scores of 0.63 and 0.59, respectively. PCMCI-based and PCMCI$^+$-based methods follow closely behind. As expected, the F1-scores slightly decrease in the \textit{3 regimes} scenario due to the increased complexity of the task. Compared with the equal regime size setting, the F1-score of RCBNB-MB decreases slightly but still remains the highest among the evaluated methods, demonstrating robustness in causal graph reconstruction even when regime sizes are unequal. Overall, methods using the Markov blanket generally outperform those relying only on parent sets, while J-PCMCI$^+$ and CD-NOD show performance only slightly better than the random baseline NegControl.

The nSHD results show a similar trend. For time series of length 600, RCBNB-MB achieves the lowest nSHD in both scenarios, with values of 0.61 for \textit{2 regimes} and 0.67 for \textit{3 regimes}. RCBNB-PA and RPCMCI$^+$-MB also obtain relatively competitive nSHD values, while most competing methods have nSHD values close to or above 1. This indicates that RCBNB-MB remains the most accurate method in terms of structural recovery under unequal regime sizes.

Table~\ref{tab:results_exp_l_inequal_1200} presents the results when the time series length is 1,200. The MER generally decreases compared to the case where the time series length is 600, suggesting that longer time series provide more information for accurate regime assignment. Although the MER of RCBNB-MB increases slightly to 0.63\%, it remains very low. The corresponding F1-scores remain competitive, reaching 0.66 in the \textit{2 regimes} scenario and 0.65 in the \textit{3 regimes} scenario. Similar trends are observed across methods: PCMCI-based and PCMCI$^+$-based approaches achieve moderate performance, with RPCMCI$^+$-MB also reaching F1-scores of 0.67 in the \textit{2 regimes} scenario and 0.65 in the \textit{3 regimes} scenario. In contrast, CASTOR, NegControl, CD-NOD, J-PCMCI$^+$, VarLiNGAM-based, and Dynotears-based methods show weaker performance.

For time series of length 1,200, the nSHD results further confirm the advantage of RCBNB-based methods. RCBNB-MB obtains the lowest nSHD in the \textit{3 regimes} scenario, with a value of 0.58, and is tied with RPCMCI$^+$-MB in the \textit{2 regimes} scenario, with a value of 0.56. Compared with the length-600 setting, the nSHD of RCBNB-MB decreases, suggesting that longer time series help improve the accuracy of graph reconstruction.

Finally, Table~\ref{tab:var_f1_l_unequal_600} and Table~\ref{tab:var_f1_l_unequal_1200} present the variances of the F1-scores for time series lengths of 600 and 1,200. Overall, the variances remain small, typically around 0.01, indicating stable performance across repetitions. Consistent with the equal regime size scenario, RPCMCI$^+$-based methods show higher variance, particularly for the PA variant, where the variance reaches approximately 0.11. In contrast, RCBNB-MB maintains low variance across all settings, further demonstrating its stability when regime sizes are unequal.

The nSHD variances are also generally small across both time series lengths. For RCBNB-MB, the nSHD variance remains between 0.02 and 0.03, indicating stable structural recovery across repetitions. Although RPCMCI$^+$-based methods show higher variance in terms of F1-score, their nSHD variance remains moderate, reaching at most 0.05. Overall, these results suggest that the structural differences observed in the mean nSHD values are consistent across repetitions.

\begin{table}[!h]
    \centering
    \caption{{\bf Unequal Regime Sizes.} The Mean Error Rate (MER), the F1-score, and the normalized Structural Hamming Distance (nSHD) across two scenarios: \textit{2 regimes} and \textit{3 regimes}. The F1-score is evaluated under three conditions: $Total$ (all edges considered), $Lagged$ (only lagged edges considered), and $Instant$ (only instantaneous edges considered). The reported values are based on 50 repetitions. The data generation process follows Equation~\ref{eq:linear_SCM}. The only difference compared to Section~\ref{subsec:simulated_data} is that the regime sizes are randomly determined, with the constraint that each regime occupies at least 10$\%$ of the time series and remains contiguous (i.e., without interruptions).}
    	\begin{subtable}[h]{1\textwidth}
        \centering
        \scriptsize
    \begin{tabular}{c|c|ccc|c|c|ccc|c}
 %         \multicolumn{9}{c}{\textit{Unequal Regime Sizes}} \\
          \hline\hline 
        & \multicolumn{5}{c|}{\textit{2 regimes}} & \multicolumn{5}{c}{\textit{3 regimes}}  \\ 
        & MER & $Total$ & $Lagged$ & $Instant$ & nSHD & MER & $Total$ & $Lagged$ & $Instant$ & nSHD \\ \hline\hline    
        RCBNB-MB & 0.61$\%$ & \textbf{0.63} & \textbf{0.70} & \textbf{0.49} & \textbf{0.61} & 3.05$\%$ & \textbf{0.59} & \textbf{0.68} & \textbf{0.41} & \textbf{0.67} \\
        RCBNB-PA & 4.57$\%$ & 0.60 & 0.69 & 0.42 & 0.66 & 7.00$\%$ & 0.54 & 0.64 & 0.33 & 0.77 \\ \hline
        RPCMCI-MB & 0.48$\%$ & 0.57 & 0.65 & $\times$  & 1.03 & 1.47$\%$ & 0.53 & 0.61 & $\times$ & 1.13 \\  
        RPCMCI-PA & 0.49$\%$ & 0.58 & 0.66 & $\times$  & 0.99 & 1.42$\%$ & 0.53 & 0.61 & $\times$ & 1.11 \\  \hline
        RPCMCI$^+$-MB & 3.07$\%$ & 0.59 & 0.66 & 0.40 & 0.67 & 4.85$\%$ & 0.52 & 0.59 & 0.34 & 0.72	\\ 
        RPCMCI$^+$-PA & 15.52$\%$ & 0.46 & 0.53 & 0.29 & 0.75 & 18.42$\%$ & 0.32 & 0.37 & 0.17 & 0.85 \\ \hline
        RVarLiNGAM-MB & 0.57$\%$ & 0.42 & 0.55 & 0.11 & 0.96 & 1.32$\%$ & 0.42 & 0.54 & 0.11 & 1.01	\\
        RVarLiNGAM-PA & 3.05$\%$ & 0.42 & 0.52 & 0.17 & 0.97 & 6.31$\%$ & 0.41 & 0.50 & 0.19 & 1.05 \\ \hline
        RDynotears-MB & 21.72$\%$ & 0.20 & 0.24 & 0.10 & 1.02 & 28.18$\%$ & 0.21 & 0.27 & 0.08 & 1.11 \\
        RDynotears-PA & 15.49$\%$ & 0.20 & 0.25 & 0.09 & 1.01 & 23.29$\%$ & 0.21 & 0.27 & 0.09 & 1.03 \\ \hline
        NegControl & $\times$ & 0.30 & 0.33 & 0.22 & 1.34 & $\times$ & 0.28 & 0.31 & 0.21 & 1.39 \\
        CD-NOD & $\times$ & 0.33 & 0.37 & 0.22 & 0.88 & $\times$ & 0.31 & 0.35 & 0.19 & 0.91 \\
        J-PCMCI$^+$ & $\times$ & 0.41 & 0.49 & 0.17 & 1.14 & $\times$ & 0.38 & 0.46 & 0.16 & 1.22 \\
        CASTOR & 26.63$\%$ & 0.14 & 0.16 & 0.07 & 0.99 & 23.10$\%$ & 0.12 & 0.15 & 0.05 & 0.99 \\
        \hline\hline 
    \end{tabular}
    \caption{Each time series has a length of 600.}
        \label{tab:results_exp_l_unequal_600}
    \end{subtable}
%\end{table}
%
\hfill
%\begin{table}[!h]
%    \centering
 %   \caption{The Mean Error Rate (MER) and the mean F1-score across two scenarios: \textit{2 regimes} and \textit{3 regimes}. The F1-score is evaluated under three conditions: $Total$ (all edges considered), $Lagged$ (only lagged edges considered), and $Instant$ (only instantaneous edges considered). The reported values are based on 50 repetitions, with each time series having a length of 1,200. The data generation process follows Equation~\ref{eq:linear_SCM}. The only difference compared to Section~\ref{subsec:simulated_data} is that the regime sizes are randomly determined, with the constraint that each regime occupies at least 10$\%$ of the time series and remains contiguous (i.e., without interruptions).}
 %   \renewcommand{\arraystretch}{1.2} 
    	\begin{subtable}[h]{1\textwidth}
        \centering
        \scriptsize
    \begin{tabular}{c|c|ccc|c|c|ccc|c}
    %      \multicolumn{9}{c}{\textit{Unequal Regime Sizes}} \\
          \hline\hline 
        & \multicolumn{5}{c|}{\textit{2 regimes}} & \multicolumn{5}{c}{\textit{3 regimes}}  \\ 
        & MER & $Total$ & $Lagged$ & $Instant$ & nSHD & MER & $Total$ & $Lagged$ & $Instant$ & nSHD \\ \hline\hline    
        RCBNB-MB & 0.63$\%$ & 0.66 & \textbf{0.76} & 0.44 & \textbf{0.56} & 0.74$\%$ & \textbf{0.65} & \textbf{0.75} & 0.44 & \textbf{0.58} \\
        RCBNB-PA & 3.00$\%$ & 0.63 & 0.72 & 0.44 & 0.63 & 3.23$\%$ & 0.62 & 0.72 & 0.42 & 0.65 \\ \hline
        RPCMCI-MB & 0.28$\%$ & 0.57 & 0.65 & $\times$ & 1.02 & 0.67$\%$ & 0.56 & 0.64 & $\times$ & 1.08 \\ 
        RPCMCI-PA & 0.26$\%$ & 0.58 & 0.66 & $\times$ & 0.99 & 0.63$\%$ & 0.57 & 0.65  & $\times$ & 1.05 \\ \hline
        RPCMCI$^+$-MB & 0.71$\%$ & \textbf{0.67} & 0.73 & \textbf{0.51} & \textbf{0.56} & 0.78$\%$ & \textbf{0.65} & 0.72 & \textbf{0.45} & 0.62 \\ 
        RPCMCI$^+$-PA & 8.32$\%$ & 0.58 & 0.63 & 0.40 & 0.66 & 9.09$\%$ & 0.53 & 0.59 & 0.34 & 0.71	\\ \hline
        RVarLiNGAM-MB & 0.25$\%$ & 0.42 & 0.55 & 0.10 & 0.93 & 0.69$\%$ & 0.43 & 0.56 & 0.11 & 0.94 \\
        RVarLiNGAM-PA & 0.47$\%$ & 0.43 & 0.55 & 0.14 & 0.93 & 3.58$\%$ & 0.42 & 0.52 & 0.15 & 0.98 \\ \hline
        RDynotears-MB & 19.83$\%$ & 0.14 & 0.17 & 0.07 & 1.00 & 33.82$\%$ & 0.17 & 0.21 & 0.08 & 1.03 \\
        RDynotears-PA & 20.40$\%$ & 0.14 & 0.18 & 0.06 & 1.00 & 30.83$\%$ & 0.17 & 0.22 & 0.07 & 1.04 \\ \hline
        NegControl & $\times$  & 0.30 & 0.33 & 0.24 & 1.34 & $\times$  & 0.29 & 0.33 & 0.22 & 1.36 \\
        CD-NOD & $\times$ & 0.34 & 0.38 & 0.21 & 0.88 & $\times$ & 0.35 & 0.40 & 0.23 & 0.86 \\
        J-PCMCI$^+$ & $\times$ & 0.42 & 0.51 & 0.19 & 1.14 & $\times$ & 0.39 & 0.47 & 0.17 & 1.26 \\
        CASTOR & 33.86$\%$ & 0.11 & 0.13 & 0.05 & 1.00 & 28.35$\%$ & 0.08 & 0.10 & 0.03 & 1.00 \\
        \hline\hline 
    \end{tabular}
    \caption{Each time series has a length of 1200.}
        \label{tab:results_exp_l_inequal_1200}
    \end{subtable}
\end{table}

\begin{table}[!h]
    \centering
    \caption{{\bf Unequal regime sizes.} Variance of the F1-score and the normalized Structural Hamming Distance (nSHD) across two scenarios: \textit{2 regimes} and \textit{3 regimes}. The variance is evaluated under three conditions: $Total$ (all edges considered), $Lagged$ (only lagged edges considered), and $Instant$ (only instantaneous edges considered). The reported values are based on 50 repetitions. The data generation process follows Equation~\ref{eq:linear_SCM}. The only difference compared to Section~\ref{subsec:simulated_data} is that the regime sizes are randomly determined, with the constraint that each regime occupies at least 10$\%$ of the time series and remains contiguous (i.e., without interruptions).}
        	\begin{subtable}[h]{1\textwidth}
            \centering
            \scriptsize
    \begin{tabular}{c|ccc|c|ccc|c}
%        \multicolumn{7}{c}{\textit{Unequal Regime Sizes}} \\
        \hline\hline 
        & \multicolumn{4}{c|}{\textit{2 regimes}} & \multicolumn{4}{c}{\textit{3 regimes}}  \\ 
        & $Total$ & $Lagged$ & $Instant$ & nSHD & $Total$ & $Lagged$ & $Instant$ & nSHD \\ \hline\hline    
        RCBNB-MB & 0.01 & 0.01 & 0.02 & 0.02 & 0.01 & 0.01 & 0.02 & 0.03 \\
        RCBNB-PA & 0.02 & 0.02 & 0.03 & 0.05 & 0.01 & 0.02 & 0.02 & 0.03 \\ \hline
        RPCMCI-MB & <0.01 & <0.01 & $\times$ & 0.04 & <0.01 & <0.01 & $\times$ & 0.02 \\  
        RPCMCI-PA & <0.01 & <0.01 & $\times$ & 0.04	& <0.01 & <0.01 & $\times$ & 0.03 \\ \hline
        RPCMCI$^+$-MB & 0.02 & 0.03 & 0.03  & 0.03 & 0.05 & 0.06 & 0.04 & 0.03 \\ 
        RPCMCI$^+$-PA & 0.07 & 0.09 & 0.05  & 0.05 & 0.08 & 0.11 & 0.04 & 0.03 \\ \hline
        RVarLiNGAM-MB & <0.01 & 0.01 & 0.01 & 0.02 & <0.01 & <0.01 & <0.01 & 0.02 \\
        RVarLiNGAM-PA & 0.01 & 0.01 & 0.02 & 0.02 & <0.01 & 0.01 & 0.01 & 0.02 \\ \hline
        RDynotears-MB & 0.01 & 0.02 & 0.01 & 0.02 & 0.01 & 0.01 & <0.01 & 0.03 \\
        RDynotears-PA & 0.01 & 0.01 & 0.01 & 0.02 & 0.01 & 0.01 & 0.01 & 0.01 \\
        \hline
        NegControl & <0.01 & 0.01 & 0.01 & 0.01 & <0.01 & <0.01 & 0.01 & 0.01 \\
        CD-NOD & 0.01 & 0.01 & 0.02 & 0.01 & 0.01 & 0.01 & 0.01 & 0.01	\\
        J-PCMCI$^+$ & 0.01 & 0.01 & 0.01 & 0.02 & <0.01 & 0.01 & 0.01 & 0.02 \\
        CASTOR & 0.01 & 0.01 & 0.01 & <0.01 & 0.01 & 0.01 & <0.01 & <0.01 \\
        \hline\hline 
    \end{tabular}
    \caption{Each time series has a length of 600.}
        \label{tab:var_f1_l_unequal_600}
\end{subtable}
% \end{table}

% \begin{table}[!h]
%     \centering
%    \caption{Variance of the F1-score across two scenarios: \textit{2 regimes} and \textit{3 regimes}. The variance is evaluated under three conditions: $Total$ (all edges considered), $Lagged$ (only lagged edges considered), and $Instant$ (only instantaneous edges considered). The reported values are based on 50 repetitions, with each time series having a length of 1,200. The data generation process follows Equation~\ref{eq:linear_SCM}. The only difference compared to Section~\ref{subsec:simulated_data} is that the regime sizes are randomly determined, with the constraint that each regime occupies at least 10$\%$ of the time series and remains contiguous (i.e., without interruptions).}
\begin{subtable}[h]{1\textwidth}
    \centering
    \scriptsize
    \begin{tabular}{c|ccc|c|ccc|c}
    %    \multicolumn{7}{c}{\textit{Unequal Regime Sizes}} \\
        \hline\hline 
        & \multicolumn{4}{c|}{\textit{2 regimes}} & \multicolumn{4}{c}{\textit{3 regimes}}  \\ 
        & $Total$ & $Lagged$ & $Instant$ & nSHD & $Total$ & $Lagged$ & $Instant$ & nSHD \\ \hline\hline    
        RCBNB-MB & 0.01 & 0.01 & 0.03 & 0.03 & 0.01 & 0.01 & 0.02 & 0.02 \\
        RCBNB-PA & 0.01 & 0.01 & 0.02 & 0.04 & 0.01 & 0.01 & 0.02 & 0.03 \\ \hline
        RPCMCI-MB & <0.01 & <0.01 & $\times$ & 0.03	& <0.01 & <0.01 & $\times$ & 0.03 \\  
        RPCMCI-PA & <0.01 & <0.01 & $\times$ & 0.04 & <0.01 & <0.01 & $\times$ & 0.03 \\ \hline
        RPCMCI$^+$-MB & 0.02 & 0.02 & 0.03 & 0.02 & 0.01 & 0.01 & 0.02 & 0.02 \\ 
        RPCMCI$^+$-PA & 0.06 & 0.06 & 0.05 & 0.05 & 0.06 & 0.07 & 0.04 & 0.03 \\ \hline
        RVarLiNGAM-MB & <0.01 & 0.01 & 0.01 & 0.01 & <0.01 & <0.01 & 0.01 & 0.01 \\
        RVarLiNGAM-PA & <0.01 & 0.01 & 0.01 & 0.01 & 0.01 & 0.01 & 0.01 & 0.01 \\ \hline
        RDynotears-MB & 0.01 & 0.02 & 0.01 & 0.01 & 0.01 & 0.01 & 0.01 & 0.01 \\
        RDynotears-PA & 0.01 & 0.02 & 0.01 & 0.01 & 0.01 & 0.01 & 0.01 & 0.02 \\
        \hline
        NegControl & <0.01 & 0.01 & 0.01 & 0.02 & <0.01 & <0.01 & 0.01 & 0.01 \\
        CD-NOD & 0.01 & 0.02 & 0.03 & 0.01 & 0.01 & 0.01 & 0.01 & 0.01 \\
        J-PCMCI$^+$ & <0.01 & 0.01 & 0.01 & 0.02 & <0.01 & 0.01 & 0.01 & 0.02 \\
        CASTOR & 0.01 & 0.01 & 0.01 & <0.01 & 0.01 & 0.01 & <0.01 & <0.01 \\
        \hline\hline 
    \end{tabular}
    \caption{Each time series has a length of 1200.}
       \label{tab:var_f1_l_unequal_1200}
\end{subtable}
\end{table}

\clearpage
\subsection{Simulated Data with 15 Variables}
\label{appendix:15_variables}
In this section, we evaluate the methods in a scenario with a larger number of variables. The data generation process follows Equation~\ref{eq:linear_SCM}, using the same parameter distributions as in Section~\ref{subsec:simulated_data}, with 15 variables and a time series of length 2,000 containing 2 contiguous regimes of unequal sizes, each occupying at least 10\% of the time series. In this setting, we do not impose the constraint on the number of shared edges, since two random graphs with this density necessarily share a larger number of edges.

Table~\ref{tab:results_exp_15var} reports the results. In terms of MER, RCBNB-MB (0.07\%), RPCMCI-MB (0.04\%), RPCMCI-PA (0.19\%), and RVarLiNGAM-MB (0.10\%) assign almost all timestamps to the correct regime, whereas RPCMCI$^+$-PA (39.16\%) and CASTOR (54.44\%) struggle considerably. Within each family of methods, the variant using the Markov blanket attains an error lower than or equal to that of the variant using only parents, confirming that the Markov blanket provides a more robust representation for regime assignment in higher dimensions. Regarding the F1 score, RCBNB-MB achieves the best performance under all three conditions (0.49 in total, 0.56 for lagged edges, and 0.45 for instantaneous edges), followed by RPCMCI-MB and RPCMCI-PA (0.43) and RPCMCI$^+$-MB (0.34). The absolute F1 levels are lower than in the setting with six variables, reflecting the increased difficulty of the task. Strikingly, the methods relying on VarLiNGAM, Dynotears, CASTOR, and CD-NOD fall below the random baseline NegControl (0.25), while J-PCMCI$^+$ (0.27) remains only marginally above it, showing that increasing the dimensionality is far more damaging to these baselines than to RCBNB-MB.

The nSHD results further confirm this trend. RCBNB-MB achieves the lowest nSHD, with a value of 0.77, followed closely by RCBNB-PA with 0.79 and RPCMCI$^+$-MB with 0.83. This indicates that RCBNB-based methods recover graph structures closest to the ground truth even in the higher-dimensional setting. In contrast, most competing methods obtain nSHD values close to or above 1, reflecting larger structural discrepancies. 

Finally, Table~\ref{tab:results_var_15var} shows that the variances of the F1 scores remain small, typically below 0.01. Consistent with the other scenarios, the methods based on RPCMCI$^+$ exhibit the highest variability (up to 0.03), while RCBNB-MB maintains a variance of at most 0.01 under all conditions, further demonstrating its stability as the dimensionality increases.

The nSHD variances are also very small, with most values below or around 0.01. In particular, RCBNB-MB has an nSHD variance below 0.01, showing that its structural recovery remains stable across repetitions. However, the low variance of weaker baselines should be interpreted together with their mean nSHD values, since it may also indicate consistently poor structural recovery rather than accurate reconstruction.

\begin{table}[!ht]
    \centering
    \scriptsize
    \caption{The Mean Error Rate (MER), the F1-score and the normalized Structural Hamming Distance (nSHD) in the \textit{2 regimes} scenario with 15 variables and unequal regime sizes. The F1-score is evaluated under three conditions: $Total$ (all edges considered), $Lagged$ (only lagged edges considered), and $Instant$ (only instantaneous edges considered). The reported values are based on 50 repetitions, with each time series having a length of 2,000.}
           	\begin{subtable}[h]{1\textwidth}
            \centering
    \begin{tabular}{c|c|ccc|c}
        \multicolumn{6}{c}{\textit{Unequal Regime Sizes, 15 Variables}} \\
        \hline\hline
        & \multicolumn{5}{c}{\textit{2 regimes}} \\
        & MER & $Total$ & $Lagged$ & $Instant$ & nSHD \\ \hline\hline
        RCBNB-MB & 0.07$\%$ & \textbf{0.49} & \textbf{0.56} & \textbf{0.45} & \textbf{0.77} \\
        RCBNB-PA & 3.08$\%$ & 0.42 & 0.40 & 0.44 & 0.79 \\ \hline
        RPCMCI-MB & 0.04$\%$ & 0.43 & 0.41 & $\times$ & 1.19 \\
        RPCMCI-PA & 0.19$\%$ & 0.43 & 0.41 & $\times$ & 1.20 \\ \hline
        RPCMCI$^+$-MB & 9.99$\%$ & 0.34 & 0.35 & 0.30 & 0.83 \\
        RPCMCI$^+$-PA & 39.16$\%$ & 0.21 & 0.23 & 0.15 & 0.94 \\ \hline
        RVarLiNGAM-MB & 0.10$\%$ & 0.02 & 0.03 & 0.01 & 1.01 \\
        RVarLiNGAM-PA & 7.55$\%$ & 0.05 & 0.05 & 0.06 & 1.03 \\ \hline
        RDynotears-MB & 7.39$\%$ & 0.06 & 0.07 & 0.04 & 1.03 \\
        RDynotears-PA & 9.95$\%$ & 0.06 & 0.07 & 0.04 & 1.03 \\
        \hline
        NegControl & $\times$ & 0.25 & 0.29 & 0.19 & 1.44 \\
        CD-NOD & $\times$ & 0.16 & 0.13 & 0.21 & 0.96 \\
        J-PCMCI$^+$ & $\times$ & 0.27 & 0.33 & 0.08 & 1.06 \\
        CASTOR & 54.44$\%$ & 0.05 & 0.06 & 0.04 & 1.02 \\
        \hline\hline
    \end{tabular}
    \caption{Mean over 50 repetitions.}
        \label{tab:results_exp_15var}
\end{subtable}

       	\begin{subtable}[h]{1\textwidth}
        \centering
        \scriptsize
    \begin{tabular}{c|ccc|c}
         \multicolumn{5}{c}{\textit{Unequal Regime Sizes, 15 Variables}} \\
        \hline\hline
        & \multicolumn{4}{c}{\textit{2 regimes}} \\
        & $Total$ & $Lagged$ & $Instant$ & nSHD\\ \hline\hline
        RCBNB-MB & <0.01 & 0.01 & 0.01 & <0.01 \\
        RCBNB-PA & <0.01 & 0.01 & 0.01 & 0.01 \\ \hline
        RPCMCI-MB & <0.01 & <0.01 & $\times$ & 0.01 \\
        RPCMCI-PA & <0.01 & <0.01 & $\times$ & 0.01 \\ \hline
        RPCMCI$^+$-MB & 0.02 & 0.02 & 0.02 & 0.01 \\
        RPCMCI$^+$-PA & 0.03 & 0.03 & 0.02 & 0.01 \\ \hline
        RVarLiNGAM-MB & <0.01 & <0.01 & <0.01 & <0.01 \\
        RVarLiNGAM-PA & <0.01 & <0.01 & <0.01 & <0.01 \\ \hline
        RDynotears-MB & <0.01 & <0.01 & <0.01 & <0.01 \\
        RDynotears-PA & <0.01 & <0.01 & <0.01 & <0.01 \\
        \hline
        NegControl & <0.01 & <0.01 & <0.01 & <0.01 \\
        CD-NOD & <0.01 & <0.01 & 0.01 & <0.01 \\
        J-PCMCI$^+$ & <0.01 & <0.01 & <0.01 & <0.01 \\
        CASTOR & <0.01 & <0.01 & <0.01 & <0.01 \\
        \hline\hline
    \end{tabular}
    \caption{Variance over 50 repetitions.}
    \label{tab:results_var_15var}
\end{subtable}
\end{table}

\clearpage
\section{Real data: some details} \label{app:realdata}
We analyze eight time series collected from an IT monitoring system with a one-minute sampling rate, provided by EasyVista. These time series capture various activities within the system. PMDB represents the extraction of some information about the messages received by the Storm ingestion system; MDB refers to an activity of a process that orients messages to other processes with respect to different types of messages; CMB represents the activity of extraction of metrics from messages; MB represents the activity of insertion of data in a database; LMB reflects the updates of the last values of metrics in Cassandra; RTMB represents the activity of searching to merge data with information coming from the check message bolt; GSIB represents the activity of insertion of historical status in database; ESB represents the activity of writing data in Elasticsearch.

\section{Source Code Information}
\label{appendix:source_code}
All methods used in Section~\ref{sec:exper} were implemented using publicly available Python libraries:

\begin{itemize}
    \item \textbf{CBNB}: \url{https://github.com/ckassaad/Hybrids_of_CB_and_NB_for_Time_Series}
    \item \textbf{VarLiNGAM / CD-NOD} (Causal-learn package): \url{https://github.com/py-why/causal-learn/tree/main}
    \item \textbf{Dynotears} (CausalNex package): \url{https://github.com/mckinsey/causalnex}
    \item \textbf{J-PCMCI$^+$ / PCMCI / PCMCI$^+$} (Tigramite package): 
          \url{https://github.com/jakobrunge/tigramite/}
    \item \textbf{NegControl}: \url{https://github.com/annennenne/negcontrol-disco}
    \item \textbf{CASTOR}: \url{https://github.com/arahmani1/CASTOR}
\end{itemize}